\documentclass{article}

\usepackage[preprint]{neurips_2026}

\usepackage[utf8]{inputenc}
\usepackage[T1]{fontenc}
\usepackage{hyperref}
\usepackage{url}
\usepackage{booktabs}
\usepackage{amsfonts}
\usepackage{nicefrac}
\usepackage{microtype}
\usepackage{xcolor}
\usepackage{amsmath}
\usepackage{amssymb}
\usepackage{amsthm}
\usepackage{mathtools}
\usepackage{bm}
\usepackage{graphicx}
\usepackage{enumitem}
\usepackage{multirow}

\usepackage{colortbl}
\definecolor{PolicyBest}{RGB}{255,246,204}
\definecolor{PolicySecond}{RGB}{224,240,255}
\newcommand{\policybest}[1]{\cellcolor{PolicyBest}#1}
\newcommand{\policysecond}[1]{\cellcolor{PolicySecond}#1}

\newtheorem{theorem}{Theorem}

\newtheorem{proposition}{Proposition}
\newtheorem{assumption}{Assumption}

\newcommand{\norm}[1]{\left\lVert #1 \right\rVert}
\newcommand{\E}{\mathbb{E}}
\newcommand{\Prob}{\mathbb{P}}

\title{Black-Mamba: Biologically-Inspired Leaky Accumulation for Conceptual Knowledge under Distribution Drift}

\author{%
  Giuseppe Soriano\\
  University of Pisa\\
  National Research Council (CNR)\\
  Pisa, Italy \\
  \texttt{giuseppe.soriano@phd.unipi.it} \\
  \And
  Nicola Tonellotto \\
  University of Pisa \\
  Pisa, Italy \\
  \texttt{nicola.tonellotto@unipi.it} \\
  \And
  Alberto Gotta \\
  National Research Council (CNR) \\
  Pisa, Italy \\
  \texttt{alberto.gotta@cnr.it}
}

\begin{document}

\maketitle

\begin{abstract}
Forecasting under real-world conditions is inherently non-stationary, as the conditional distribution of future observations evolves over time. Recent test-time adaptive sequence models address this challenge by updating internal states during inference, but tie adaptation to instantaneous prediction errors or surprise. This coupling can conflate persistent distribution shift with stochastic innovations, leading to unnecessary updates and inefficient adaptation. 
We introduce \emph{Black-Mamba}, a test-time adaptive forecasting architecture that formulates online adaptation as evidence-gated state tracking under distribution drift. The model augments a base predictor with a dynamic memory updated when temporally accumulated surprisal provides sufficient evidence of a regime change. This turns adaptation into a selective, event-driven process rather than a continuous one.
Across multiple forecasting benchmarks with non-stationary dynamics, \emph{Black-Mamba} achieves competitive or improved predictive performance compared to existing test-time adaptation methods while significantly reducing the number of memory updates during inference. 
Together with mathematical analysis and biological evidence, these results suggest that accumulated surprisal provides a principled signal for distinguishing persistent drift from transient noise, yielding more efficient and robust adaptation.
\end{abstract}

\section{Introduction}
\label{sec:introduction}

Time-series forecasting is often presented as a supervised learning problem over a fixed dataset, but in many practical settings, the target law is itself time-varying. Electricity demand changes as weather regimes and human behavior evolve, vehicular traffic flows drift with infrastructure and social rhythms, Internet traffic distribution changes over time driven by the dynamic interplay of emerging applications (e.g., video, AI-driven services), evolving user behavior patterns, protocol and platform shifts, physiological signals change across latent clinical states, and financial or industrial processes are shaped by structural changes that are only partially observed. In such settings, the main difficulty is not only modeling long-range dependencies but also tracking an evolving environment.

This observation has produced two major lines of work. The former improves the \emph{inductive bias} of sequence models. Patch-based Transformers show that much of the perceived failure of attention in forecasting originates from suboptimal tokenization and channel mixing rather than from attention itself \citep{2023_timeseriesworth_nie}. State-space and selective recurrent models emphasize temporal bias, linear complexity, and stability for long contexts \citep{2024_simbasimplifiedmambabased_patro}. More recent multivariate architectures explicitly model time and variate dimensions as distinct but coupled axes, arguing that two-dimensional inductive bias is essential for multivariate forecasting \citep{2025_effectivelydesigning2dimensional_cao,2025_letomodelingmultivariate_behrouz,2025_hydradualexponentiated_meskin}. This line of work is indispensable, but it still largely assumes that a sufficiently well-designed static model can explain the sequence.

The latter moves from static inference to \emph{online adaptation}. Continual test-time adaptation methods update model parameters from unlabeled target streams in order to follow non-stationary domains \citep{2022_continualtesttimedomain_wang}. Test-time training layers reinterpret the hidden state as a trainable model whose parameters are updated during inference \citep{2025_learninglearntest_sun}. Titans and follow-up neural-memory architectures argue that long-term memory should be explicitly written at test time, allowing the predictor to react to context in a way that standard finite-window attention cannot \citep{2024_titanslearningmemorize_behrouz,2025_atlaslearningoptimally_behrouz}. This line is attractive because it acknowledges the moving-target nature of the problem. However, most existing formulations are driven by instantaneous loss or surprise, and therefore implicitly assume that every sufficiently large prediction error is evidence of persistent drift.

That assumption is too strong for non-stationary forecasting. Prediction error is a mixture of structured drift and unpredictable innovation. If the model updates its long-term memory at every time step, it reacts not only to persistent changes in the conditional law, but also to outliers, transient fluctuations, and irreducible noise. As a consequence, online plasticity can become variance amplification, error accumulation, or catastrophic forgetting rather than true adaptation \citep{2022_continualtesttimedomain_wang,2025_empiricalstudycatastrophic_luo}. The key unresolved question is therefore not whether adaptation should happen, but \emph{when}.

Our starting point is that forecasting under drift should be cast as \emph{online state estimation}. We model the environment through a latent regime state that evolves over time and changes the conditional distribution of the next observation. A memory-augmented predictor is then useful not because it blindly reacts to surprise, but because its memory can serve as an implicit estimator of that latent regime. This perspective suggests that the update mechanism should be selective: it should respond to persistent evidence of distributional change, while suppressing zero-mean innovations.

This paper proposes \emph{Black-Mamba}, an event-triggered memory architecture for inference-time adaptation. The biological intuition comes from two classical ideas in the neuroscience of memory. First, long-term potentiation shows that durable memory traces arise from sufficiently strong or repeated activation rather than from every local perturbation \citep{1973_longlastingpotentiationsynaptic_bliss,2004_longtermpotentiationmemory_lynch,2017_briefhistorylongterm_nicoll}. Second, synaptic tagging and capture separates transient eligibility from durable consolidation: local events create the possibility of memory, but persistence requires temporally coordinated reinforcement \citep{2011_makingmemorieslast_redondo,2014_taggingcapturehypothesis_viola,2021_memoryconsolidationimprovement_luboeinski}. We combine this memory-consolidation viewpoint with the leaky integrate-and-fire principle, where evidence accumulates with decay and only threshold crossings trigger discrete events \citep{2006_reviewintegrateandfireneuron_burkitt,2005_adaptiveexponentialintegrateandfire_brette,2002_spikingneuronmodels_gerstner}.

The statistical intuition is parallel. If drift admits a decomposition into a predictable component and an innovation component, then temporally integrated evidence can enhance the signal-to-noise ratio: persistent deviation accumulates coherently, while martingale-like noise remains centered and bounded. This yields a natural architectural principle. Instead of updating memory directly from the instantaneous loss, we first filter a scalar surprisal signal through a leaky integrator and update memory only when the integrated evidence exceeds a threshold. In this way, adaptation is no longer coupled to every local mismatch, but only to deviations that persist long enough to be statistically meaningful.

The paper makes three contributions. First, it formalizes non-stationary forecasting as prediction over a time-indexed family of conditional laws, with a precise distinction between explained variation and genuine unresolved drift. Second, it introduces \emph{Black-Mamba}, a memory-augmented architecture whose update mechanism is event-triggered by surprisal accumulation rather than by instantaneous loss. Third, it provides a mathematical proof that, under explicit martingale-noise assumptions, evidence-gated updating suppresses innovation noise, accumulates sustained drift, and controls false-trigger probabilities, thereby supporting its use as a general principle for learning at inference time.

\section{Problem Formulation}
\label{sec:problem}

Let $\mathcal{T} \subseteq \mathbb{N}$ denote time. At time $t \in \mathcal{T}$ we observe information $\mathcal{I}(t)$ and then a target variable $Y(t+1) \in \mathcal{Y}$ is generated according to a conditional law
\[
Y(t+1) \sim P_t, \qquad P_t := P(\cdot \mid \mathcal{I}(t)).
\]
The available information may be the recent history of the same process, as in univariate autoregressive forecasting, or a vector of multivariate covariates including lags, exogenous variables, calendar features, and metadata. The learning objective is to predict $Y(t+1)$ from $\mathcal{I}(t)$.

The key departure from standard formulations is that we do not assume $P_t$ to be constant. Instead, we consider a time-indexed family of conditional distributions
\[
P : \mathcal{T} \to \mathcal{P}(\mathcal{Y}),
\]
and focus on the non-stationary regime in which $t \mapsto P_t$ evolves over time. A useful interpretation is that the conditional law is modulated by a latent environment state $z(t) \in \mathcal{Z}$ such that
\[
P_t = P(\cdot \mid \mathcal{I}(t), z(t)),
\qquad
z(t+1) = F(z(t)) + \eta(t),
\]
where $F$ captures structured regime dynamics and $\eta(t)$ is an innovation term. In this view, the optimal predictor
\[
f_t^\star(\mathcal{I}(t)) := \E[Y(t+1)\mid \mathcal{I}(t)]
\]
is itself time-varying.

\paragraph{Explained variation versus true drift.}
Not every temporal change should be interpreted as unresolved drift. If the relevant regime information can already be reconstructed from the observed inputs, then temporal variation is \emph{explained} by conditioning and should not trigger adaptation. This distinction is important for time series because many recurring effects, such as diurnal or weekly cycles, disappear once the proper covariates are included. We reserve the term \emph{distribution drift} for residual time variation in the conditional law after conditioning on the available information.

To quantify drift, define the formal increment
\[
\Delta(t) := P_{t+1} - P_t,
\]
whose magnitude can be measured through a discrepancy $d(P_{t+1},P_t)$ such as total variation, Kullback--Leibler divergence, or Wasserstein distance. The central statistical issue is whether the observed temporal variation contains a predictable component or is dominated by innovation noise.

Let $\mathcal{H}_t$ denote the observable history up to time $t$. We decompose drift as
\[
\Delta(t) = \Delta_{\mathrm{pred}}(t) + \Delta_{\mathrm{noise}}(t),
\]
where
\[
\Delta_{\mathrm{pred}}(t) := \E[\Delta(t)\mid \mathcal{H}_t],
\qquad
\E[\Delta_{\mathrm{noise}}(t)\mid \mathcal{H}_t] = 0.
\]
This decomposition is the conceptual core of the paper. The forecasting problem is meaningful only when the cumulative predictable component is large enough, over task-relevant horizons, to rise above the concentration scale of the innovation process. Otherwise no adaptation mechanism can reliably distinguish drift from noise.

This viewpoint also clarifies the limitation of standard loss-driven updates. If the prediction loss at time $t$ is
\[
\ell_t := \ell\!\left(f_t(\mathcal{I}(t)), Y(t+1)\right),
\]
then $\ell_t$ is influenced by both $\Delta_{\mathrm{pred}}(t)$ and $\Delta_{\mathrm{noise}}(t)$. An update rule that reacts directly to $\ell_t$ or to its gradient therefore lacks an explicit mechanism to separate persistent drift from transient innovations.

\section{Black-Mamba Architecture}
\label{sec:architecture}

\subsection{Base predictor and adaptive memory}

We consider a predictor of the form
\[
\hat{Y}(t+1) = f_{\theta,\phi_t}(\mathcal{I}(t)),
\]
where $\theta$ are static backbone parameters and $\phi_t$ is a dynamic memory state adapted during inference. This follows the general memory-augmented view of modern adaptive sequence models \citep{2024_titanslearningmemorize_behrouz,2025_learninglearntest_sun,2025_atlaslearningoptimally_behrouz}: the backbone provides a stable representation and the memory tracks the evolving local regime.

The architectural interpretation is that $\phi_t$ acts as an implicit estimator of the latent state $z(t)$. In the idealized case,
\[
\phi_t \approx h(z(t))
\]
for some task-dependent representation $h$, so adapting $\phi_t$ online is equivalent to tracking the latent environment. Unlike standard Titans-style memory, however, Black-Mamba does not write to memory at every time step.

\subsection{From surprisal to evidence}

Let $s_t$ denote a scalar surprisal signal derived from the current prediction error. We keep the construction abstract and write
\[
s_t = \psi(\ell_t),
\]
where $\psi$ may be the identity, a normalized loss, or another scalar statistic computed from the prediction residual. We assume that this signal itself admits a decomposition
\[
s_t = a_t + \xi_t,
\]
where $a_t$ is the component induced by predictable drift and $\xi_t$ is an innovation term satisfying
\[
\E[\xi_t \mid \mathcal{H}_{t-1}] = 0.
\]

The key architectural move is to decouple \emph{local surprisal} from \emph{memory commitment}. We introduce a leaky evidence accumulator
\[
u_t = \lambda u_{t-1} + s_t,
\qquad 0 < \lambda < 1,
\]
which integrates recent evidence while discounting stale information. The memory is updated only when the accumulated evidence exceeds a threshold:
\[
\text{if } u_t \ge \tau, \quad
\phi_{t+1} = \mathcal{U}\!\left(\phi_t, \nabla_\phi \ell_t\right),
\]
and otherwise
\[
\phi_{t+1} = \phi_t.
\]
After an update, the accumulator is reset or partially decayed:
\[
u_t \leftarrow \rho u_t, \qquad 0 \le \rho < 1.
\]
For concreteness, we instantiate the memory update as a local gradient step,
\[
\mathcal{U}\!\left(\phi_t, \nabla_\phi \ell_t\right)
=
\phi_t - \eta \nabla_\phi \ell_t,
\]
but any local test-time adaptation operator can be used in place of this step.

\paragraph{Architecture summary.}
Black-Mamba is therefore a three-stage system:
\begin{enumerate}[leftmargin=1.2em,itemsep=0.2em]
    \item a backbone predictor $f_{\theta,\phi_t}$ that outputs the next-step forecast;
    \item a surprisal encoder $\psi(\ell_t)$ that converts instantaneous error into a scalar evidence signal;
    \item a leaky event trigger that decides whether the current evidence justifies a persistent memory update.
\end{enumerate}
This architecture is intentionally minimal. It preserves the strengths of memory-augmented forecasting models while changing only the \emph{policy of when memory is written}.

\section{Foundations of Temporal Integration}
\label{sec:theory}

The event-triggered update rule is justified by a simple statistical principle: persistent drift should accumulate, whereas zero-mean innovations should not.

Unrolling the accumulator yields
\[
u_t
=
\sum_{k=1}^{t} \lambda^{t-k} s_k
=
\sum_{k=1}^{t} \lambda^{t-k} a_k
+ \sum_{k=1}^{t} \lambda^{t-k} \xi_k.
\]
The first term aggregates structured deviation; the second aggregates innovation noise through a stable linear filter. Under the standard assumption that $\xi_t$ is a martingale-difference sequence with bounded conditional variance, the filtered innovation remains centered and uniformly controlled.

\begin{proposition}[Filtered innovations remain bounded]
\label{prop:filtered-noise}
Assume $(\xi_t)$ satisfies
\[
\E[\xi_t \mid \mathcal{H}_{t-1}] = 0,
\qquad
\E[\xi_t^2 \mid \mathcal{H}_{t-1}] \le \sigma^2
\]
almost surely. Define
\[
w_t := \sum_{k=1}^{t} \lambda^{t-k} \xi_k,
\qquad 0 < \lambda < 1.
\]
Then $\E[w_t \mid \mathcal{H}_{t-1}] = 0$ and
\[
\mathrm{Var}(w_t) \le \frac{\sigma^2}{1-\lambda^2}.
\]
\end{proposition}

The proposition implies that temporal integration does not cause innovation noise to drift arbitrarily. By contrast, if the structured component has a persistent positive mean, then the signal term accumulates to a non-zero steady-state scale.

\begin{proposition}[Persistent drift accumulates]
\label{prop:signal}
If $a_t \approx \mu > 0$ over a window sufficiently longer than the memory horizon $(1-\lambda)^{-1}$, then
\[
\sum_{k=1}^{t} \lambda^{t-k} a_k \approx \frac{\mu}{1-\lambda}.
\]
Hence the signal grows relative to the centered innovation component as persistence increases.
\end{proposition}

\paragraph{False-trigger control.}
Beyond bounding the variance of the filtered innovation, the same martingale view controls spurious update events. In the no-drift case, if the calibrated surprisal satisfies $s_t=\xi_t$ with $|\xi_t|\le b$, Azuma--Hoeffding gives
\[
\Prob(u_t\ge \tau)
\le
\exp\!\left(
-\frac{\tau^2(1-\lambda^2)}{2b^2}
\right).
\]
Thus the threshold has a direct statistical interpretation: it controls the probability that pure innovation noise triggers a memory write. A full derivation and the connection to sequential change detection are given in Appendix~\ref{app:theory}.

Propositions~\ref{prop:filtered-noise}--\ref{prop:signal} explain why updating on thresholded accumulated surprisal is preferable to updating on instantaneous error. The leaky accumulator acts as a low-pass filter: outliers and transient innovations may perturb $u_t$, but they do not usually push it across threshold unless they are followed by corroborating deviations. Gradual drift, by contrast, may produce only weak instantaneous gradients, yet still accumulates until it becomes statistically distinguishable.

This logic also provides a principled interpretation of the biological analogy. The instantaneous loss corresponds to a local plasticity signal, not to durable memory by itself. The accumulator $u_t$ plays the role of a latent evidence reservoir, and the threshold crossing marks the transition from transient eligibility to persistent consolidation. In other words, \emph{Black-Mamba} combines the statistical idea of separating predictable drift from innovation noise with the neurophysiological idea that lasting memory should be committed only after temporally integrated evidence.

Finally, note that standard loss-driven test-time memory models are recovered as a limiting case. In the boundary case $\lambda=0$ and $\tau \le \inf_t s_t$, the accumulator reduces to the instantaneous surprisal $u_t=s_t$ and the trigger fires at every step; \emph{Black-Mamba} therefore reduces to a purely reactive memory update. The proposed architecture generalizes existing approaches by introducing a continuous spectrum between always-update and selectively-update regimes.

\section{Experiments}
\label{sec:experiments}

We evaluate Black-Mamba on public long-term forecasting benchmarks and controlled synthetic drift streams. The real suite follows the benchmark organization used by the PatchTST framework \citep{2023_timeseriesworth_nie}, which is also the protocol adopted in the SiMBA evaluation; this allows us to compare the frozen SiMBA backbone and its adaptive variants under the same data organization and forecasting setting. The real datasets include the ETT family, Electricity, Exchange Rate, Traffic, and Weather, while the synthetic suite contains gradual, recurrent, regime-based, noisy, and selective-update drift streams. All datasets use chronological 0.45/0.10/0.45 train/validation/test splits, with an intentionally long test stream to evaluate whether the held-out distribution exhibits drift relative to the training regime. Real forecasting experiments use context length 96 and prediction horizons 96, 192, 336, and 720; synthetic experiments use horizon 1. Unless explicitly stated otherwise, MSE and MAE refer to standardized target-space errors following the PatchTST evaluation convention, so that losses remain comparable across datasets and channels with different physical scales.

The central adaptive comparison is designed to isolate the update policy. We use a shared frozen SiMBA forecasting backbone \citep{2024_simbasimplifiedmambabased_patro} and attach the same Low-Rank Adaptation (LoRA) forecasting-head adapter \citep{2021_loralowrankadaptation_hu} to both adaptive variants. The continuous-update baseline implements the inference-time memory-writing mechanism described in Titans-inspired neural-memory work \citep{2024_titanslearningmemorize_behrouz}; the official architecture was not public at the time of implementation, so we reproduce the relevant continuous update rule. Black-Mamba uses the same backbone and the same LoRA adapter, but replaces continuous writing with the proposed event-triggered update policy.

This implementation is intentionally lightweight. LoRA keeps the online state small, confines inference-time learning to a low-rank adapter rather than the full backbone, and makes the update rule portable to other forecasting architectures. Our goal is therefore not to establish this particular LoRA head as the only or optimal instantiation of Black-Mamba, but to test the update policy itself: whether selective surprisal accumulation can retain the practical benefit of inference-time adaptation, reduce the number of online updates, and potentially make adaptation more robust to noisy or transient drift signals, as predicted by the theoretical analysis.

\subsection{Benchmark Accuracy}
\label{sec:benchmark-accuracy}

Tables~\ref{tab:real-main-avg} and~\ref{tab:synthetic-main} report the forecasting losses. On real long-term forecasting benchmarks, the frozen SiMBA backbone is already highly competitive, and both adaptive variants change the averaged losses only marginally at benchmark-table precision. This should not be read as a failure of adaptation: the real benchmark suite is not uniformly drifting, and our drift diagnostics identify material concept drift only in Exchange Rate and Traffic. Horizon-averaged benchmark losses therefore mostly test whether the adaptive mechanism preserves the base forecaster under ordinary benchmark conditions. Black-Mamba passes this test: it remains practically indistinguishable from the frozen and continuously updated SiMBA variants on most real datasets.

The synthetic suite provides the complementary regime, where concept drift is controlled by construction. Here, both adaptive methods improve strongly over frozen SiMBA across all scenarios, confirming that inference-time adaptation is useful when the target rule changes. Continuous updating is usually slightly more accurate than Black-Mamba, but the average advantage is small: $0.299\%$ MSE and $0.398\%$ MAE. This is the first empirical sign of the trade-off predicted by the theory. If every observation carries useful drift evidence, continuous updates can extract a small additional gain; when many observations are redundant or noisy, a selective policy should retain most of the benefit while avoiding unnecessary writes.

\begin{table*}[t]
\centering
\caption{Real benchmark results averaged over horizons 96, 192, 336, and 720. Lower is better. Bold and underline mark the best and second-best results across all models; light yellow and light blue mark the best and second-best results among SiMBA, Cont. update, and Black-Mamba, isolating update-policy effects for the chosen frozen baseline. Ties are highlighted within the corresponding rank.}
\label{tab:real-main-avg}
\small
\resizebox{\textwidth}{!}{
\begin{tabular}{l*{8}{cc}}
\toprule
Dataset
& \multicolumn{2}{c}{Autoformer}
& \multicolumn{2}{c}{Informer}
& \multicolumn{2}{c}{DLinear}
& \multicolumn{2}{c}{Linear}
& \multicolumn{2}{c}{NLinear}
& \multicolumn{2}{c}{SiMBA}
& \multicolumn{2}{c}{Cont. update}
& \multicolumn{2}{c}{Black-Mamba} \\
\cmidrule(lr){2-3}\cmidrule(lr){4-5}\cmidrule(lr){6-7}\cmidrule(lr){8-9}
\cmidrule(lr){10-11}\cmidrule(lr){12-13}\cmidrule(lr){14-15}\cmidrule(lr){16-17}
& MSE & MAE & MSE & MAE & MSE & MAE & MSE & MAE
& MSE & MAE & MSE & MAE & MSE & MAE & MSE & MAE \\
\midrule
ETTh1 & 0.777 & 0.613 & 1.186 & 0.797 & \underline{0.627} & \underline{0.531} & \textbf{0.624} & \textbf{0.528} & 0.671 & 0.541 & \policybest{0.668} & \policybest{0.547} & \policybest{0.668} & \policybest{0.547} & \policybest{0.668} & \policybest{0.547} \\
ETTh2 & 0.439 & 0.452 & 3.746 & 1.546 & 0.565 & 0.514 & 0.562 & 0.506 & \underline{0.402} & \underline{0.417} & \policybest{\textbf{0.392}} & \policybest{\textbf{0.410}} & \policybest{\textbf{0.392}} & \policybest{\textbf{0.410}} & \policybest{\textbf{0.392}} & \policybest{\textbf{0.410}} \\
ETTm1 & 0.758 & 0.592 & 0.847 & 0.650 & \textbf{0.520} & \textbf{0.464} & \textbf{0.520} & \underline{0.465} & 0.545 & 0.475 & \policybest{\underline{0.538}} & \policybest{0.477} & \policybest{\underline{0.538}} & \policybest{0.477} & \policybest{\underline{0.538}} & \policybest{0.477} \\
ETTm2 & 0.278 & 0.351 & 1.089 & 0.802 & 0.303 & 0.374 & 0.303 & 0.368 & \textbf{0.250} & \textbf{0.322} & \policybest{\underline{0.257}} & \policybest{\underline{0.333}} & \policybest{\underline{0.257}} & \policybest{\underline{0.333}} & \policybest{\underline{0.257}} & \policybest{\underline{0.333}} \\
Electricity & 0.267 & 0.364 & 0.590 & 0.540 & \underline{0.224} & 0.303 & \underline{0.224} & 0.303 & 0.228 & \underline{0.295} & \policybest{\textbf{0.206}} & \policybest{\textbf{0.281}} & \policybest{\textbf{0.206}} & \policybest{\textbf{0.281}} & \policybest{\textbf{0.206}} & \policybest{\textbf{0.281}} \\
Exchange Rate & 0.945 & 0.696 & 4.367 & 1.683 & 1.250 & 0.728 & 1.228 & 0.717 & \textbf{0.659} & \textbf{0.536} & \policybest{\underline{0.761}} & \policybest{\underline{0.571}} & \policysecond{0.766} & \policybest{\underline{0.571}} & \policybest{\underline{0.761}} & \policybest{\underline{0.571}} \\
Traffic & \underline{0.646} & \underline{0.410} & 0.851 & 0.461 & 0.769 & 0.488 & 0.716 & 0.462 & 0.701 & 0.445 & \policybest{\textbf{0.456}} & \policybest{\textbf{0.285}} & \policybest{\textbf{0.456}} & \policybest{\textbf{0.285}} & \policybest{\textbf{0.456}} & \policybest{\textbf{0.285}} \\
Weather & 0.653 & 0.539 & 0.626 & 0.526 & \textbf{0.602} & \underline{0.503} & \textbf{0.602} & \underline{0.503} & 0.637 & 0.511 & \policybest{\underline{0.613}} & \policybest{\textbf{0.495}} & \policybest{\underline{0.613}} & \policybest{\textbf{0.495}} & \policybest{\underline{0.613}} & \policybest{\textbf{0.495}} \\
\bottomrule
\end{tabular}
}
\end{table*}

\begin{table*}[t]
\centering
\caption{Synthetic drift benchmark results at horizon 1. Lower is better.}
\label{tab:synthetic-main}
\small
\resizebox{\textwidth}{!}{
\begin{tabular}{l*{7}{cc}}
\toprule
Model
& \multicolumn{2}{c}{L-Grad.}
& \multicolumn{2}{c}{L-Grad. noisy}
& \multicolumn{2}{c}{Rec. seas.}
& \multicolumn{2}{c}{Rec. seas. noisy}
& \multicolumn{2}{c}{Regime}
& \multicolumn{2}{c}{Regime noisy}
& \multicolumn{2}{c}{Selective stress} \\
\cmidrule(lr){2-3}\cmidrule(lr){4-5}\cmidrule(lr){6-7}\cmidrule(lr){8-9}
\cmidrule(lr){10-11}\cmidrule(lr){12-13}\cmidrule(lr){14-15}
& MSE & MAE & MSE & MAE & MSE & MAE & MSE & MAE
& MSE & MAE & MSE & MAE & MSE & MAE \\
\midrule
SiMBA & 1.945 & 1.141 & 1.668 & 1.070 & 0.476 & 0.434 & 0.616 & 0.583 & 1.912 & 1.100 & 1.715 & 1.071 & 1.283 & 0.669 \\
Black-Mamba & 0.519 & 0.553 & 0.768 & 0.639 & 0.365 & 0.401 & 0.536 & 0.531 & 0.564 & 0.579 & 0.690 & 0.610 & 1.033 & 0.656 \\
Cont. update & 0.517 & 0.551 & 0.765 & 0.637 & 0.362 & 0.399 & 0.534 & 0.525 & 0.562 & 0.577 & 0.688 & 0.609 & 1.034 & 0.655 \\
\bottomrule
\end{tabular}
}
\end{table*}

\subsection{Selective Updates and the Theoretical Trade-off}
\label{sec:selective-updates-theory}

The theory suggests that inference-time adaptation should not be treated as a binary choice between a frozen model and continuous learning. Continuous updates are optimal only under the implicit assumption that each new error provides reliable evidence about the current regime. Black-Mamba relaxes this assumption by accumulating surprisal with leak and committing an update only after persistent evidence crosses a threshold. The empirical prediction is therefore not that Black-Mamba must always beat continuous updating in raw loss, but that it should approximate its adaptation benefit with fewer writes and greater resistance to transient or noisy drift signals.

The results match this prediction. Across the seven synthetic streams, Black-Mamba uses only $52.19\%$ of the continuous-update budget, saving $1{,}506{,}107$ updates out of $3{,}149{,}986$, while losing only a fraction of a percent in average error relative to continuous updating. It captures $98.9\%$ of the MSE improvement obtained by continuous updating over frozen SiMBA. This is the main empirical support for the proposed update policy: the benefit of online adaptation is largely preserved, but the number of inference-time parameter writes is almost halved.

The segment-level diagnostics further support the mechanism. Black-Mamba does not simply update at a uniformly lower rate; its update rate rises where the stream becomes harder. Across synthetic segments, the correlation between segment MSE and Black-Mamba update rate is $+0.819$, and the hardest segment quintile receives $1.87\times$ more updates than the easiest quintile. This is the behavior expected from leaky surprisal accumulation: isolated errors decay, while persistent mismatch accumulates into an update event.

\begin{figure*}[t]
\centering
\begin{minipage}{0.49\textwidth}
\centering
\includegraphics[width=\linewidth]{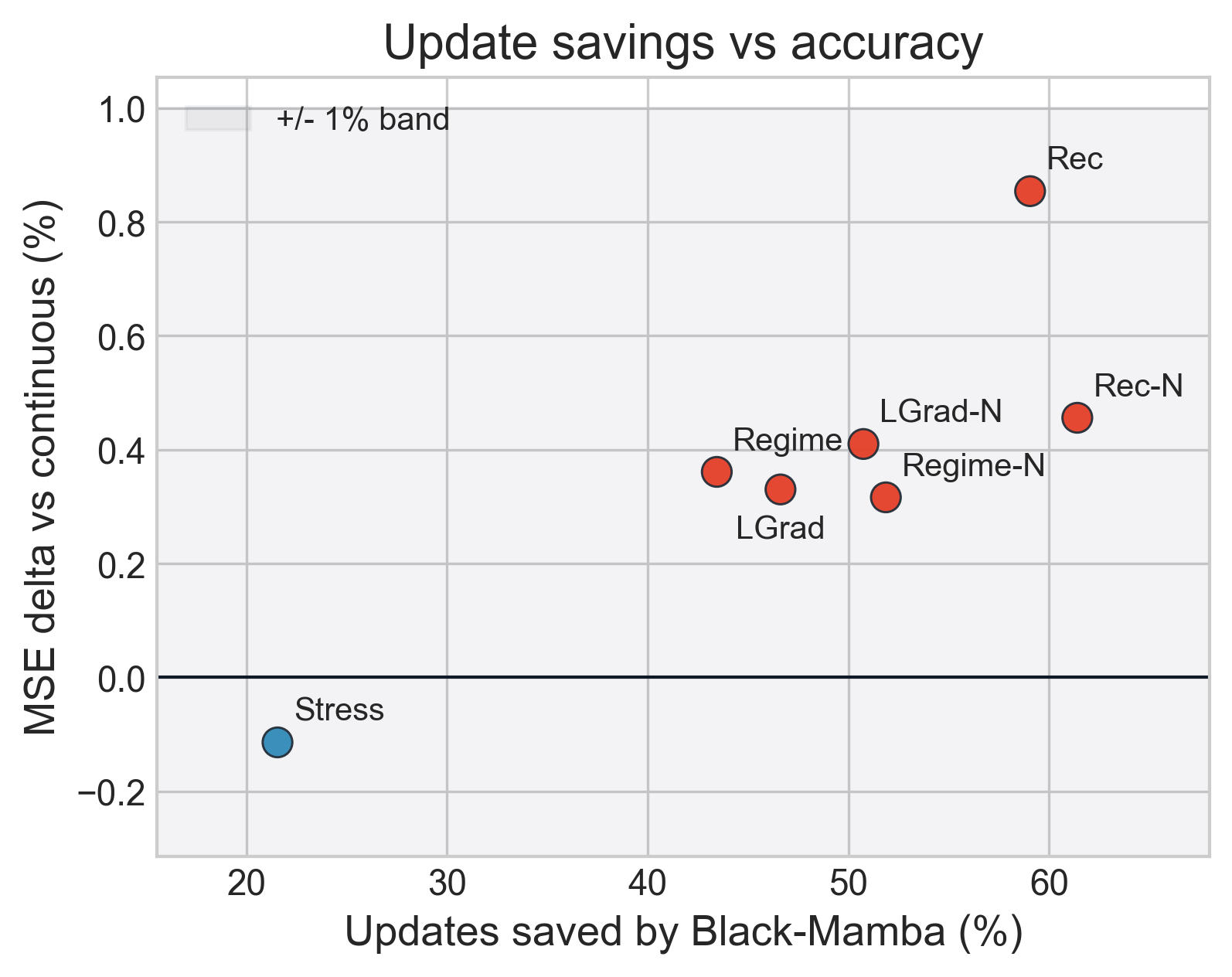}
\end{minipage}
\hfill
\begin{minipage}{0.49\textwidth}
\centering
\includegraphics[width=\linewidth]{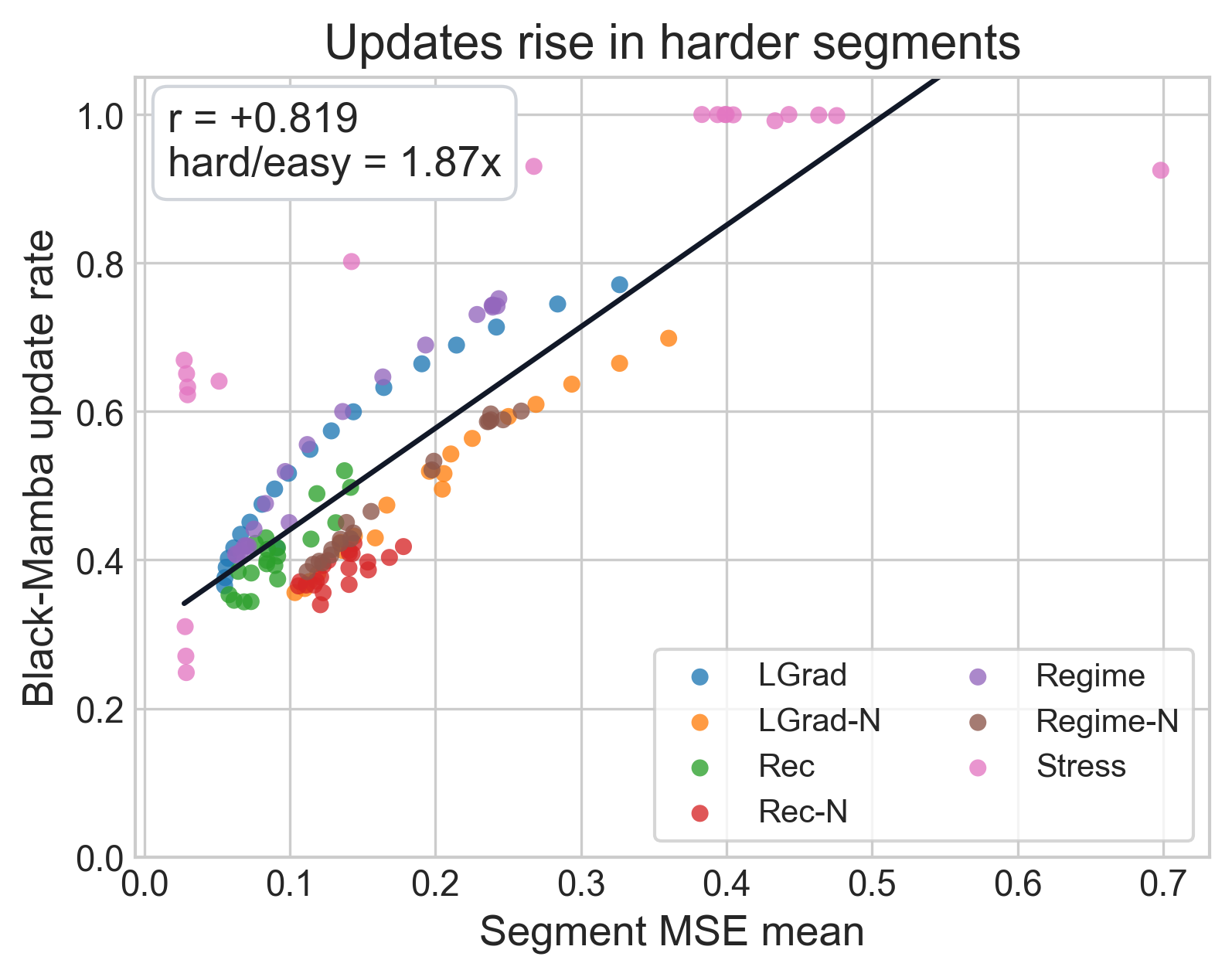}
\end{minipage}
\caption{Synthetic trace dynamics. Left: Black-Mamba achieves large update savings with only marginal MSE changes relative to the continuous update adapter. Right: Black-Mamba update rate increases with segment difficulty, indicating selective adaptation rather than uniformly weaker adaptation.}
\label{fig:update-efficiency}
\end{figure*}

\subsection{Real Drift Case Study}
\label{sec:real-drift-case-study}

Exchange Rate is the most informative real benchmark because it is the strongest real-drift dataset in our diagnostics. Table~\ref{tab:exchange-policy-main} therefore reports the adaptive-policy comparison by horizon rather than only as an aggregate. Black-Mamba updates only $7.5$--$10.4\%$ of online steps on Exchange Rate, whereas the continuous-update adapter writes at every step. At horizons 96, 192, and 336, the differences are small and not statistically significant. At horizon 720, however, continuous updating is worse than Black-Mamba by $+1.797\%$ MSE and $+0.556\%$ MAE.

This result is important because it is exactly the failure mode targeted by Black-Mamba. Long-horizon Exchange Rate forecasting exposes the model to persistent distribution shift, but also to noisy local innovations that should not all be written into memory. Continuous updating reacts to every step, making it more vulnerable to over-adaptation on the test stream. Black-Mamba instead requires accumulated evidence before committing an update, and in this real-drift long-horizon case it avoids the degradation observed under continuous writing while using roughly one tenth of the update rate.

Traffic provides a useful contrast and is reported fully in the appendix. Although some paired tests detect a consistent advantage for continuous updating on Traffic, the effect size is microscopic, on the order of $10^{-3}\%$ or less. We therefore interpret Traffic as practical equivalence rather than meaningful superiority of continuous adaptation. Together, Exchange Rate and Traffic support a conservative real-data conclusion: Black-Mamba is at least practically equivalent to continuous updating on real drift benchmarks, and in the strongest long-horizon drift case it is measurably better.

\begin{table}[t]
\centering
\caption{Per-horizon adaptive-policy comparison on Exchange Rate, the strongest real-drift benchmark. $\Delta=100\cdot(\mathrm{Cont.}-\mathrm{BM})/\mathrm{BM}$; positive values favor Black-Mamba. $p_{\mathrm{dir}}$ is the HAC one-sided p-value in the direction of the observed mean difference.}
\label{tab:exchange-policy-main}
\small
\begin{tabular}{llrrrrr}
\toprule
Dataset & Horizon & BM update rate & $\Delta$ MSE & $p_{\mathrm{dir}}$ & $\Delta$ MAE & $p_{\mathrm{dir}}$ \\
\midrule
\multirow{4}{*}{Exchange Rate}
& 96  & 7.5\%  & -0.685\% & 0.370 & -0.599\% & 0.270 \\
& 192 & 10.4\% & +0.087\% & 0.420 & -0.160\% & 0.208 \\
& 336 & 9.9\%  & -0.926\% & 0.283 & -0.286\% & 0.371 \\
& 720 & 9.1\%  & +1.797\% & 0.056 & +0.556\% & 0.039 \\
\bottomrule
\end{tabular}
\end{table}

\subsection{Statistical Evidence}
\label{sec:statistical-evidence}

All adaptive comparisons are paired on matched online forecasting windows. For real multi-step horizons, we use one-sided HAC/Newey--West tests on paired loss differences, setting the lag to $H-1$ to account for overlapping $H$-step target windows. Moving-block bootstrap intervals are reported in the appendix as robustness checks. The main text reports only the p-values needed for the central real-drift claim.

On Exchange Rate at horizon 720, the degradation of continuous updating relative to Black-Mamba is statistically significant for MAE ($p=0.039$) and borderline for MSE ($p=0.056$). Other Exchange Rate horizons are not significant. These results are consistent with the theory: selective updating is not expected to dominate continuous updating everywhere, but it should reduce unnecessary writes and avoid degradation when instantaneous errors are a noisy proxy for persistent drift. Full per-horizon results for Exchange Rate, Traffic, and synthetic streams are provided in Appendix~\ref{app:additional-results}.

\section{Conclusion}
\label{sec:conclusion}

Black-Mamba reframes inference-time adaptation as selective evidence-gated memory commitment. Rather than updating continuously from instantaneous prediction error, it accumulates surprisal with leak and writes to a lightweight LoRA head adapter only when persistent evidence supports an update. The empirical results support the intended trade-off: Black-Mamba preserves practical accuracy while reducing update frequency, and in the strongest real-drift long-horizon case it avoids the degradation observed under continuous updating. On streams without measurable drift, Black-Mamba leaves the backbone performance essentially unchanged, as desired, which is important since adaptation should not pay an accuracy cost when it is unnecessary. When drift is present, the comparison with the continuous-update competitor shows the intended difference: continuous updating can extract small marginal gains in dense synthetic drift, but Black-Mamba preserves nearly the same practical accuracy with far fewer memory writes and avoids degradation in the strongest real-drift long-horizon case. The method is intentionally modular: because adaptation is implemented as a LoRA adapter on a frozen forecasting head, the update policy can be transferred to other backbones with minimal architectural changes.

\paragraph{Limitations and broader impact.}
The main limitation of this study is the limited availability of real forecasting benchmarks that exhibit measurable concept drift and in which the effect of online adaptation is large enough to be clearly observed. This is visible in our experiments: the benefits of adaptation are most evident on the synthetic drift streams, whereas the real datasets provide fewer cases with strong, measurable drift. At the same time, synthetic data can favor continuous updating, because the drift process is controlled, dense, and deliberately constructed to be learnable from the stream. This makes the practical difference between selective and continuous adaptation harder to isolate on real-world benchmarks. Nevertheless, the combination of theoretical analysis, synthetic drift experiments, and the real-drift evidence on Exchange Rate and Traffic supports the central claim: accumulated surprisal can act as a principled criterion for deciding when inference-time learning should occur. Future work should therefore expand the evaluation to domains where concept drift is both frequent and measurable, such as network traffic, information retrieval, recommender systems, industrial monitoring, clinical or wearable sensing, etc. A complementary direction is to collect or construct new streaming benchmarks specifically designed to evaluate learning under concept drift. The positive impact of Black-Mamba is a more efficient and general inference-time adaptation mechanism: in the limiting case where the trigger always fires, the proposed framework recovers the continuous-update behavior of Titans-inspired neural-memory models, while in the selective regime it reduces unnecessary updates. The principle is not tied to forecasting or to the LoRA implementation used here; LoRA provides a lightweight and immediately usable instantiation, but the same evidence-gated update rule can be applied to other learning-at-inference-time architectures whenever online plasticity must be controlled.

\newpage
\bibliographystyle{plainnat}
\bibliography{references}


\newpage
\appendix

\section{Broader Context for Memory-Augmented Forecasting}
\label{app:context}

The argument developed in the main paper is deliberately narrow: it focuses on the problem of \emph{when} a test-time memory should be updated under non-stationary conditions. This appendix places that argument in a broader modeling landscape, because the usefulness of an event-triggered update policy depends on how one positions the proposal relative to the current architecture literature on time series, sequence modeling, and adaptive inference.

A first observation is that adaptive memory is not the only way in which recent work has attempted to address the difficulty of non-stationary or long-range forecasting. A substantial portion of the literature argues that the central issue is architectural misalignment between standard sequence models and the inductive structure of time series. PatchTST, for example, shows that when time series are segmented into subseries-level patches and channels are modeled independently, Transformer-style models recover strong performance in long-term forecasting \citep{2023_timeseriesworth_nie}. The significance of this result is conceptual as much as empirical: it suggests that part of the challenge attributed to attention may actually be a challenge of tokenization and representation. In a similar spirit, SiMBA proposes that state-space sequence modeling combined with explicit channel mixing can narrow the gap between subquadratic models and strong attention-based baselines \citep{2024_simbasimplifiedmambabased_patro}. These works imply that any theory of online adaptation must remain compatible with the possibility that part of apparent drift is simply poorly modeled structure.

The literature becomes even more instructive in the multivariate case. Typhon, LETO, and Hydra all emphasize that multivariate time series are not just longer univariate sequences, but objects with two coupled axes: temporal dynamics and cross-variate interactions \citep{2025_effectivelydesigning2dimensional_cao,2025_letomodelingmultivariate_behrouz,2025_hydradualexponentiated_meskin}. Their common claim is that forecasting models underperform when they either erase temporal inductive bias or assume too naively that cross-variate structure should always be mixed in the same way. What matters for the present paper is that these models also report long-horizon error propagation in purely reactive recurrent systems. This observation supports our central concern: when the environment changes over time, unrestricted online updates can turn adaptation into instability.

From this point of view, the recent memory-augmented literature can be interpreted as a second, complementary response. Titans introduces a neural long-term memory that is explicitly written at test time using surprise-driven updates \citep{2024_titanslearningmemorize_behrouz}. Test-time training layers generalize this idea by making the hidden state itself a trainable object whose parameters are updated during inference \citep{2025_learninglearntest_sun}. ATLAS, Memory Caching, and related developments then argue that once one accepts the premise of inference-time adaptation, the next challenge is to make the memory larger, better optimized, or less myopic \citep{2025_atlaslearningoptimally_behrouz,2026_memorycachingrnns_behrouz}. In this family of models, the central difficulty is not only how to represent long context, but how to \emph{manage} it online.

At the same time, the theoretical interpretation of these architectures remains unsettled. Some authors describe them literally as systems that memorize context at test time, whereas others argue that at least part of their behavior can be reduced to learned forms of linear attention or recurrent optimization \citep{2026_testtimetrainingkv_liu,2025_expressivenesssoftmaxattention_mongaras,2025_transformersintrinsicoptimizers_ren}. Hybrid architectures such as Mamba--Transformer combinations or sequence-level switching systems make this ambiguity even more visible by showing that the practical boundary between attention, recurrence, and memory is increasingly fluid \citep{2025_understandingenhancingmambatransformer_lee,2026_transmambasequencelevelhybrid_li}. For our purposes, this ambiguity is not a problem. The event-triggering principle proposed in the main paper does not depend on whether a model is ultimately interpreted as memorization or as implicit attention. It depends only on the claim that some internal state is adapted online, and that such adaptation should be selective rather than indiscriminate.

Finally, there is a broader systems-level discussion about whether continuously updated neural memory is sufficient for stable long-term retention. Cognitive work on memory-augmented language models emphasizes episodic segmentation, selective encoding, and competitive retrieval \citep{2025_largelanguagemodels_donga}. Engineering systems such as retrofitted external memory or parametric memory banks provide different operational responses \citep{2026_prometheusmindretrofitting_wind,2026_movemixturevalue_li}. A stronger critique goes further and argues that shared continuous parameter memory may face a fundamental interference barrier under dense episodic storage, which would make discrete addressing unavoidable in some regimes \citep{2026_attentionnotretention_zahn}. Although our setting is forecasting rather than factual episodic memory, this discussion is still relevant. It reinforces the idea that memory should not be written eagerly; once memory is treated as a limited and interference-prone resource, the question of when to update becomes a first-order modeling issue rather than a secondary implementation detail.

\section{Detailed Problem Setting}
\label{app:problem}

This section expands the formal framing used in the main paper. The objective is not simply to restate notation, but to make precise why the forecasting problem considered here is better interpreted as one of tracking an evolving conditional law than of estimating a single static function.

Let $\mathcal{T} \subseteq \mathbb{N}$ denote time. Time is not treated as a random variable, but as the deterministic index along which the environment evolves. At each instant $t \in \mathcal{T}$, the learner has access to information $\mathcal{I}(t)$ and wishes to predict the next target value $Y(t+1)$. In a univariate autoregressive formulation, $\mathcal{I}(t)$ may coincide with the recent history of the same process,
\[
\mathcal{I}(t)=\mathcal{H}_t := (Y(1),\dots,Y(t-1)).
\]
In a multivariate formulation, $\mathcal{I}(t)$ can include lagged observations, exogenous variables, calendar features, metadata, or any other side information. These two formulations look different at the level of representation, but conceptually they are variants of the same conditional prediction problem.

The key modeling decision is to define the next-step law as
\[
Y(t+1)\sim P_t,
\qquad
P_t:=P(\cdot \mid \mathcal{I}(t)).
\]
This immediately induces a trajectory of conditional distributions $t \mapsto P_t$. Stationarity corresponds to the special case where this map is constant. In the present work we focus instead on the more realistic non-stationary regime in which $P_t$ changes over time. One may then introduce a latent regime state $z(t)$ and write
\[
P_t = P(\cdot \mid \mathcal{I}(t), z(t)),
\qquad
z(t+1)=F(z(t))+\eta(t),
\]
where $F$ captures structured regime evolution and $\eta(t)$ denotes unpredictable perturbations. This formulation is useful because it turns non-stationarity into a structured problem: the environment drifts because an underlying regime evolves, not because the data are assumed to be arbitrarily non-stationary.

An important refinement concerns the distinction between \emph{explained variation} and \emph{unresolved drift}. In forecasting practice, many temporal effects are perfectly genuine but should not be treated as drift if they are already encoded in the observed inputs. For example, daily or weekly seasonality may disappear once hour-of-day or day-of-week are included among the covariates. In that case the conditional law is stable after conditioning, even though the raw sequence is strongly time-dependent. By contrast, if the environment changes in ways that remain unobserved after conditioning on $\mathcal{I}(t)$, then the learner faces genuine residual drift. The present paper is explicitly about this latter case.

To describe the temporal change of the predictive law, the main text introduced the formal increment
\[
\Delta(t)=P_{t+1}-P_t.
\]
Of course, this notation is symbolic: in practice one measures the magnitude of drift through a discrepancy on probability measures. Common choices include total variation,
\[
d_{\mathrm{TV}}(\mu,\nu)=\sup_{A \in \mathcal{B}(\mathcal{Y})} |\mu(A)-\nu(A)|,
\]
Kullback--Leibler divergence,
\[
D_{\mathrm{KL}}(\mu\|\nu)=\int \log\!\left(\frac{d\mu}{d\nu}\right)d\mu,
\]
and the $p$-Wasserstein distance,
\[
W_p(\mu,\nu)
=
\left(
\inf_{\pi \in \Pi(\mu,\nu)}
\int \rho(y,y')^p\, d\pi(y,y')
\right)^{1/p}.
\]
The choice among these is application-dependent. Total variation captures worst-case disagreement on measurable events, KL divergence emphasizes directional likelihood mismatch, and Wasserstein distance is often more natural when the output space has an intrinsic geometry.

The most important conceptual step, however, is not the metric choice but the decomposition of drift into a predictable and an innovation component. Let $\mathcal{H}_t$ denote the observable filtration. Then one may write
\[
\Delta(t)=\Delta_{\mathrm{pred}}(t)+\Delta_{\mathrm{noise}}(t),
\qquad
\Delta_{\mathrm{pred}}(t):=\E[\Delta(t)\mid \mathcal{H}_t],
\]
with
\[
\E[\Delta_{\mathrm{noise}}(t)\mid \mathcal{H}_t]=0.
\]
This decomposition should be understood as the forecasting analogue of the predictable--martingale split from stochastic-process theory. Its significance is direct: adaptation is justified only to the extent that the learner can extract and track the predictable component. The residual innovation should not be memorized, because by definition it has no stable conditional mean structure available from the observable past.

\section{Formal Theoretical Foundations of the Problem Formulation}
\label{app:theorems}

The main paper only uses the minimum amount of formal machinery needed to motivate the architecture. This section gives the more complete theoretical backbone that originally motivated the formulation. The goal is twofold. First, it shows that the distinction between predictable drift and innovation noise is not heuristic, but follows from standard results in probability theory. Second, it clarifies the precise sense in which temporal integration can improve the identifiability of structured drift.

\subsection{Conditional expectation as the predictable component}

The first relevant fact is that conditional expectation is the optimal predictor, in mean-square sense, of any square-integrable random variable given the available information.

\begin{theorem}[Conditional expectation as an $L^2$ projection]
\label{thm:l2_projection}
Let $(\Omega,\mathcal{F},\Prob)$ be a probability space, let $\mathcal{G}\subseteq\mathcal{F}$ be a sub-$\sigma$-algebra, and let $U\in L^2(\Omega,\mathcal{F},\Prob)$. Then $\E[U\mid \mathcal{G}]$ is the unique element of $L^2(\Omega,\mathcal{G},\Prob)$ such that
\[
\E\!\left[(U-\E[U\mid\mathcal{G}])V\right]=0
\qquad
\text{for all } V\in L^2(\Omega,\mathcal{G},\Prob),
\]
and equivalently
\[
\E[U\mid\mathcal{G}]
=
\arg\min_{V\in L^2(\Omega,\mathcal{G},\Prob)}
\E[(U-V)^2].
\]
\end{theorem}

In our setting, this theorem justifies the definition of the predictable component of drift. If $\Delta(t)$ denotes the drift variable and $\mathcal{H}_t$ the observable history, then $\E[\Delta(t)\mid \mathcal{H}_t]$ is the best mean-square summary of the drift available from past information. In other words, if there is any component of drift that is in principle learnable from the history, it is represented by that conditional expectation. This is why we define
\[
\Delta_{\mathrm{pred}}(t):=\E[\Delta(t)\mid \mathcal{H}_t],
\qquad
\Delta_{\mathrm{noise}}(t):=\Delta(t)-\E[\Delta(t)\mid \mathcal{H}_t].
\]
The residual is not merely an error term; it is orthogonal to every square-integrable function of the available history.

\subsection{Predictable dynamics and innovation structure}

The decomposition above is also consistent with the classical structure theorem for adapted processes.

\begin{theorem}[Doob decomposition]
\label{thm:doob}
Let $(X_t)_{t\ge 0}$ be an integrable adapted process with respect to a filtration $(\mathcal{F}_t)_{t\ge 0}$. Then there exist a predictable process $(A_t)_{t\ge 0}$ with $A_0=0$ and a martingale $(M_t)_{t\ge 0}$ with $M_0=X_0$ such that
\[
X_t = M_t + A_t
\qquad
\text{for all } t\ge 0.
\]
This decomposition is unique up to indistinguishability.
\end{theorem}

The role of Theorem~\ref{thm:doob} in the present paper is interpretive. We do not require the drift process itself to be modeled explicitly as a martingale plus predictable process in every application, but the theorem explains why this is the natural language for online adaptation. The predictable component is the part that can, in principle, be tracked. The martingale component is the part that can generate volatility in realized errors without carrying stable directional information. A memory-update mechanism that reacts identically to both is therefore misaligned with the statistical structure of the problem.

\subsection{When drift is statistically distinguishable}

To make the problem non-vacuous, one needs more than a decomposition. One also needs the innovation term to be sufficiently controlled so that sustained predictable drift is not drowned by fluctuations.

\begin{assumption}[Bounded conditional variance of innovations]
\label{ass:variance}
There exists a predictable sequence $(\sigma_t^2)_{t\ge 0}$ such that
\[
\E\!\left[\norm{\Delta_{\mathrm{noise}}(t)}^2 \mid \mathcal{H}_t\right]
\le \sigma_t^2
\qquad
\text{almost surely.}
\]
\end{assumption}

Under this assumption, concentration inequalities for martingale sums become relevant. The classical result that best matches our use case is Freedman's inequality.

\begin{theorem}[Freedman's inequality]
\label{thm:freedman}
Let $(M_t,\mathcal{F}_t)_{t\ge 0}$ be a real-valued martingale with $M_0=0$, and let the martingale differences be
\[
\xi_t := M_t - M_{t-1}.
\]
Assume that $|\xi_t|\le b$ almost surely for all $t$, and define the predictable quadratic variation
\[
V_t := \sum_{s=1}^{t}\E[\xi_s^2\mid \mathcal{F}_{s-1}].
\]
Then for all $x,v>0$,
\[
\Prob\!\left(M_t \ge x \ \text{and}\ V_t \le v\right)
\le
\exp\!\left(-\frac{x^2}{2(v+bx)}\right).
\]
\end{theorem}

Theorem~\ref{thm:freedman} is not used in the main paper as a technical tool, but as a conceptual bridge. It says that cumulative innovation has a concentration scale. Therefore, the forecasting problem becomes meaningful exactly when the cumulative predictable component of drift exceeds that scale over time windows relevant to the task. This leads to the following informal non-degeneracy condition: over a window of length $W$, the magnitude of
\[
\sum_{s=t}^{t+W-1}\Delta_{\mathrm{pred}}(s)
\]
should be comparable to or larger than the typical fluctuation size of
\[
\sum_{s=t}^{t+W-1}\Delta_{\mathrm{noise}}(s).
\]
If this condition fails, then no event-triggering mechanism, however sophisticated, can reliably distinguish signal from noise.

\subsection{Temporal integration as asymptotic noise suppression}

The event-triggered architecture in the main paper is based on a scalar surprisal signal $s_t$. We write
\[
s_t = a_t + \xi_t,
\]
where $a_t$ denotes the structured component induced by predictable drift and $\xi_t$ is again a martingale-difference innovation. The basic idea that motivates integration is that averaging suppresses martingale differences.

\begin{theorem}[Strong law for martingale differences]
\label{thm:slln_martingale}
Let $(\xi_t)$ be a martingale-difference sequence with respect to $(\mathcal{H}_t)$ such that
\[
\E[\xi_t\mid \mathcal{H}_{t-1}] = 0,
\qquad
\sup_t \E[\xi_t^2] < \infty.
\]
Then
\[
\frac{1}{t}\sum_{k=1}^{t}\xi_k \to 0
\qquad \text{almost surely.}
\]
\end{theorem}

The practical limitation of Theorem~\ref{thm:slln_martingale} is that uniform averaging is too inertial for drifting environments. If one gives equal weight to all past evidence, adaptation becomes too slow to follow evolving regimes. This is the reason for introducing a leaky integrator instead.

\section{Statistical Rationale for Event-Triggered Updates}
\label{app:theory}

The architecture in the main paper is motivated by a simple but consequential principle: if the update signal combines structured drift and innovation noise, then one should not allow every local fluctuation to write into memory. This appendix expands that argument and clarifies why a leaky integrator is the natural filtering device.

\begin{center}
\fbox{%
\begin{minipage}{0.92\linewidth}
\textbf{Algorithm 1: Evidence-gated test-time adaptation.}

\vspace{0.3em}
\textbf{Input:} model $f_{\theta,\phi}$, loss $\ell$, surprisal encoder $\psi$, accumulator $u=0$, leak $\lambda$, threshold $\tau$, reset factor $\rho$, update operator $\mathcal{U}$.

\vspace{0.3em}
\begin{enumerate}[leftmargin=1.4em,itemsep=0.15em]
    \item At time $t$, predict $\hat{Y}(t+1)=f_{\theta,\phi_t}(\mathcal{I}(t))$.
    \item Observe $Y(t+1)$ and compute $\ell_t=\ell(\hat{Y}(t+1),Y(t+1))$.
    \item Encode surprisal $s_t=\psi(\ell_t)$ and update evidence $u\leftarrow \lambda u+s_t$.
    \item If $u\ge \tau$, commit memory $\phi_{t+1}=\mathcal{U}(\phi_t,\nabla_\phi \ell_t)$ and set $u\leftarrow \rho u$.
    \item Otherwise, keep $\phi_{t+1}=\phi_t$.
\end{enumerate}
\end{minipage}}
\end{center}

Algorithm~1 shows the architecture-independent form of the mechanism. The update operator may be a gradient step, a neural-memory write, a recurrent-state refresh, or another local test-time adaptation rule. The invariant part is the separation between evidence accumulation and memory commitment.

Suppose the model produces a scalar surprisal statistic $s_t$ from the current loss, and suppose this statistic can be decomposed as
\[
s_t=a_t+\xi_t,
\]
where $a_t$ represents the structured component induced by persistent misalignment with the current regime, and $\xi_t$ is an innovation term satisfying
\[
\E[\xi_t\mid \mathcal{H}_{t-1}]=0.
\]
If one updates memory directly from $s_t$, then every realization of $\xi_t$ contributes to plasticity. In a setting with outliers, heteroskedastic noise, or abrupt but non-persistent local fluctuations, this is precisely what one would like to avoid.

The classical remedy is temporal averaging. If $(\xi_t)$ is a martingale-difference sequence with uniformly bounded second moments, then one has
\[
\frac{1}{t}\sum_{k=1}^{t}\xi_k \to 0
\quad \text{a.s.}
\]
as $t \to \infty$. This fact already indicates that repeated averaging suppresses innovation noise. Yet uniform averaging is a poor operational choice for drifting environments because it gives equal weight to stale evidence and recent evidence. Inference-time adaptation must remain local enough to follow evolving regimes. This is why the main paper adopts exponential integration rather than simple running averages.

The leaky accumulator
\[
u_t=\lambda u_{t-1}+s_t,
\qquad 0<\lambda<1,
\]
implements precisely this compromise. By unrolling the recursion one obtains
\[
u_t = \sum_{k=1}^{t}\lambda^{t-k}a_k
     + \sum_{k=1}^{t}\lambda^{t-k}\xi_k.
\]
The first term is the temporally weighted structured signal, while the second is the filtered innovation process. Under the bounded-variance assumption used in the main text, the innovation part remains centered and uniformly controlled. Hence the leaky filter enhances the signal-to-noise ratio whenever the structured component persists over a time horizon comparable to $(1-\lambda)^{-1}$.

The next proposition formalizes the boundedness of the filtered innovation component and is one of the few technical statements in this appendix for which we provide a full derivation, since it directly justifies the architecture.

\begin{proposition}[Variance bound for filtered martingale differences]
\label{prop:filtered_variance}
Let $(\xi_t)$ be a martingale-difference sequence such that
\[
\E[\xi_t^2\mid \mathcal{H}_{t-1}] \le \sigma^2.
\]
Define
\[
w_t := \sum_{k=1}^{t}\lambda^{t-k}\xi_k,
\qquad 0<\lambda<1.
\]
Then
\[
\E[w_t] = 0,
\]
and
\[
\mathrm{Var}(w_t) \le \frac{\sigma^2}{1-\lambda^2}.
\]
\end{proposition}

\begin{proof}
The mean claim follows directly from linearity and the martingale-difference property:
\[
\E[w_t]
=
\sum_{k=1}^{t}\lambda^{t-k}\E[\xi_k]
= 0.
\]
For the second moment one has
\[
\E[w_t^2]
=
\E\!\left[
\left(\sum_{k=1}^{t}\lambda^{t-k}\xi_k\right)^2
\right].
\]
Cross terms vanish because martingale differences are orthogonal in $L^2$, hence
\[
\E[w_t^2]
=
\sum_{k=1}^{t}\lambda^{2(t-k)}\E[\xi_k^2]
\le
\sigma^2 \sum_{k=1}^{t}\lambda^{2(t-k)}.
\]
The last term is a geometric sum bounded by
\[
\sigma^2 \sum_{j=0}^{\infty}\lambda^{2j}
=
\frac{\sigma^2}{1-\lambda^2}.
\]
Since $w_t$ is centered, this proves the variance bound.
\end{proof}

The proposition is the finite-memory analogue of the averaging argument in Theorem~\ref{thm:slln_martingale}. It shows that leaky integration suppresses innovation without letting its variance grow unbounded over time. In other words, the accumulator stores persistent evidence, not a random walk of noise.

The complementary signal statement is simpler but equally important. If the structured component has a persistent local mean $\mu>0$ over a time scale longer than the effective memory horizon, then
\[
\sum_{k=1}^{t}\lambda^{t-k}a_k
\approx
\frac{\mu}{1-\lambda}.
\]
Thus the structured term accumulates to a non-zero scale while the innovation remains centered and variance-controlled. This asymmetry is the exact reason for introducing a threshold over accumulated surprisal rather than acting on instantaneous loss.

This observation has three practical implications. First, gradual drift becomes visible even when instantaneous losses are individually weak. A sequence of small but coherent deviations may fail to trigger a standard reactive update, yet accumulate under the leaky filter until it crosses threshold. Second, isolated outliers become less influential because their effect decays unless corroborated by subsequent evidence. Third, the threshold parameter $\tau$ acquires a transparent interpretation: it controls how much integrated evidence is required before one is willing to treat the observed deviation as a candidate regime change rather than as noise.

\paragraph{Relation to sequential change detection.}
The trigger can also be interpreted as a lightweight form of sequential evidence accumulation. Classical change-detection procedures accumulate evidence for a distributional change and raise an alarm when the statistic crosses a threshold. Black-Mamba follows the same statistical principle, but the accumulated quantity is a model-internal surprisal signal rather than an explicit likelihood ratio, and the alarm controls an internal memory update rather than declaring an external change point. This is useful when the post-change distribution is unknown or evolves continuously, because the update rule can operate directly on the predictive mismatch already produced by the model.

\begin{proposition}[False-trigger control under calibrated innovation noise]
\label{prop:false_trigger}
Assume that there is no predictable drift after validation calibration, so the trigger input is $s_t=\xi_t$, where $(\xi_t)$ is a martingale-difference sequence with $|\xi_t|\le b$ almost surely. Let
\[
u_t=\lambda u_{t-1}+\xi_t,
\qquad
0<\lambda<1,
\qquad
u_0=0.
\]
Then for any threshold $\tau>0$,
\[
\Prob(u_t\ge \tau)
\le
\exp\!\left(
-\frac{\tau^2(1-\lambda^2)}{2b^2}
\right).
\]
\end{proposition}

\begin{proof}
Unrolling the recursion gives
\[
u_t=\sum_{k=1}^{t}\lambda^{t-k}\xi_k.
\]
This is a weighted martingale sum. Since $|\xi_k|\le b$, the weighted increments are bounded by $b\lambda^{t-k}$. Applying Azuma--Hoeffding gives
\[
\Prob(u_t\ge \tau)
\le
\exp\!\left(
-\frac{\tau^2}{2b^2\sum_{k=1}^{t}\lambda^{2(t-k)}}
\right).
\]
Because
\[
\sum_{k=1}^{t}\lambda^{2(t-k)}
\le
\frac{1}{1-\lambda^2},
\]
the stated bound follows.
\end{proof}

Proposition~\ref{prop:false_trigger} gives a direct false-update interpretation to the trigger: increasing $\tau$ or shortening the effective memory length reduces the probability that pure innovation noise commits a memory update. This complements the variance bound above by controlling threshold crossings, not only the second moment of the filtered noise.

It is also worth clarifying that the event-triggering principle is logically independent of the internal update operator. The main paper instantiates memory updates with a gradient step on $\phi_t$, but the same trigger could gate a recurrent-state refresh, a write into an external memory, a cache update, or a more elaborate optimizer over a neural memory module. What remains invariant is the separation between \emph{evidence accumulation} and \emph{memory commitment}. That separation is, in our view, the central theoretical contribution of the paper.

\section{Neurophysiological Motivation and Computational Analogy}
\label{app:neuro}

The biological motivation for \emph{Black-Mamba} should not be read as a claim of mechanistic realism. Rather, it serves to justify the architectural principle that persistent memory should be more selective than transient error correction. This appendix expands the argument that the proposed event-triggering mechanism is aligned with well-established ideas from the neuroscience of memory.

The starting point is the classical literature on long-term potentiation (LTP). From the earliest experiments of \citet{1973_longlastingpotentiationsynaptic_bliss} through later reviews \citep{2004_longtermpotentiationmemory_lynch,2017_briefhistorylongterm_nicoll}, a stable lesson emerges: durable memory traces are not formed by every perturbation. They require sufficiently strong, repeated, or appropriately coordinated activity to produce persistent changes in synaptic efficacy. This already suggests an important analogy for forecasting under drift. If every local mismatch between prediction and observation led to a durable update of long-term memory, the resulting system would be overly reactive and unstable. Biological memory does not operate in that way.

The synaptic tagging and capture literature sharpens the analogy. According to this view, a local event may set a transient tag that makes a synapse eligible for later consolidation, but persistence requires the arrival of additional plasticity-related resources within a critical time window \citep{2011_makingmemorieslast_redondo,2014_taggingcapturehypothesis_viola}. The conceptual importance of this framework lies in the explicit separation between \emph{eligibility} and \emph{consolidation}. In computational terms, instantaneous surprisal should be interpreted not as an immediate command to rewrite memory, but as an indication that the current state of the system may deserve revision if subsequent evidence supports that interpretation.

This view is reinforced by computational neuroscience. \citet{2021_memoryconsolidationimprovement_luboeinski} show that synaptic tagging and capture mechanisms in recurrent spiking networks do not merely preserve memories; they can improve later recall by stabilizing the right traces over time. This is particularly relevant for our setting because it suggests that selectivity in memory writing is not merely a strategy for reducing variance. It can also improve the functional usefulness of memory by preserving only those traces that survive temporal scrutiny.

The integrate-and-fire family of models provides the final piece of the analogy. In these models, noisy inputs are integrated over time with leak, and discrete events occur only when the resulting latent variable crosses a threshold \citep{2006_reviewintegrateandfireneuron_burkitt,2005_adaptiveexponentialintegrateandfire_brette,2002_spikingneuronmodels_gerstner}. This does not imply that the accumulator in \emph{Black-Mamba} is a literal membrane potential. The analogy is instead structural. The scalar variable $u_t$ is a latent evidence reservoir: local surprises contribute to it, old evidence decays, repeated deviations add constructively, and only sufficiently persistent evidence causes a discrete event. The thresholded update is thus the computational counterpart of selective commitment.

Seen from this perspective, the proposed architecture combines two biologically inspired principles. The first is \emph{tagging}: instantaneous discrepancy creates eligibility for change. The second is \emph{integration}: durable commitment requires accumulation over time. The result is not a biologically detailed memory model, but a computationally disciplined one. That level of abstraction is exactly what is appropriate for the present paper: we do not need to emulate the biophysics of memory in order to import the more general lesson that persistence should be earned, not assumed.

\section{Experimental and Reproducibility Details}
\label{app:experimental-details}

This section reports the experimental and reproducibility details needed to interpret and replicate the evaluation.

\subsection{Datasets, splits, and preprocessing}

We follow the public long-term forecasting benchmark protocol popularized by the PatchTST framework \citep{2023_timeseriesworth_nie}. This choice is deliberate: SiMBA is evaluated on the same benchmark family, and our implementation preserves the same data organization and forecasting protocol in order to make the comparison with the frozen SiMBA backbone and the adaptive variants as direct as possible. For all datasets, we use chronological splits with proportions 0.45/0.10/0.45 for training, validation, and test, respectively. We use this relatively large test split to make the test distribution sufficiently long for measuring whether the test stream exhibits drift relative to the training regime.

The real benchmark suite includes the following datasets. The ETT family contains four electricity-transformer temperature datasets: ETTh1 and ETTh2 are hourly series, while ETTm1 and ETTm2 are 15-minute series, all taken from the public ETDataset repository.\footnote{\url{https://github.com/zhouhaoyi/ETDataset}} Electricity is the UCI Electricity Load Diagrams dataset, containing client-level electricity consumption time series.\footnote{\url{https://archive.ics.uci.edu/ml/datasets/ElectricityLoadDiagrams20112014}} Exchange Rate is the multivariate exchange-rate benchmark distributed in the multivariate time-series data repository.\footnote{\url{https://github.com/laiguokun/multivariate-time-series-data}} Traffic is based on PeMS traffic sensor measurements from California freeway systems.\footnote{\url{https://pems.dot.ca.gov/}} Weather is built from the Jena weather archive.\footnote{\url{https://www.bgc-jena.mpg.de/wetter/}} In our drift-oriented version of Weather, we use observations from 2003-01-01 up to the last date available in our collected copy, 2026-04-07, in order to provide a long temporal span for concept-drift diagnostics. 

Real benchmark experiments use context length 96 and prediction horizons 96, 192, 336, and 720 when available. All real datasets are processed chronologically, without shuffling across time, and the validation split is used for model selection and adapter calibration before streaming evaluation on the held-out test stream.

The synthetic suite contains seven non-stationary streams designed to isolate different forms of concept drift. Each stream is generated as a controlled multivariate process with an underlying regime variable that changes the conditional relationship between recent observations and the next target. The \texttt{lgradual} stream introduces a smooth transition between regimes, while \texttt{lgradual\_noisy} adds observation noise to the same gradual drift structure. The \texttt{recurrent\_seasonal\_boundary} stream alternates between recurring seasonal decision boundaries, and \texttt{recurrent\_seasonal\_boundary\_noisy} adds noise to make boundary recovery less stable. The \texttt{regime\_plateau\_drift} stream contains extended plateaus separated by regime changes, with \texttt{regime\_plateau\_drift\_noisy} adding noise to the plateau-drift process. Finally, \texttt{selective\_update\_stress} is designed to stress the update policy: long easy regions require few updates, while later difficult regions require rapid adaptation. Synthetic forecasting uses horizon 1 and paired online traces of 449,999 windows per adaptive run.

Following the PatchTST evaluation framework \citep{2023_timeseriesworth_nie}, which is also the benchmark protocol used to evaluate SiMBA, we report losses in the standardized target space in order to make results comparable across datasets and target channels with different physical units and scales. This is the metric convention used throughout our comparison. Let $\mu_c^{\mathrm{train}}$ and $\sigma_c^{\mathrm{train}}$ denote the mean and standard deviation of target channel $c$, computed on the training split only. For a target value $y_{t,h,c}$ and prediction $\hat{y}_{t,h,c}$ at window $t$, horizon step $h$, and channel $c$, the standardized quantities are
\[
\tilde{y}_{t,h,c}
=
\frac{y_{t,h,c}-\mu_c^{\mathrm{train}}}{\sigma_c^{\mathrm{train}}},
\qquad
\tilde{\hat{y}}_{t,h,c}
=
\frac{\hat{y}_{t,h,c}-\mu_c^{\mathrm{train}}}{\sigma_c^{\mathrm{train}}}.
\]
The reported standardized MSE and MAE are therefore
\[
\mathrm{MSE}_{\mathrm{std}}
=
\frac{1}{N}
\sum_{t,h,c}
\left(
\tilde{\hat{y}}_{t,h,c}
-
\tilde{y}_{t,h,c}
\right)^2
=
\frac{1}{N}
\sum_{t,h,c}
\left(
\frac{\hat{y}_{t,h,c}-y_{t,h,c}}
{\sigma_c^{\mathrm{train}}}
\right)^2,
\]
\[
\mathrm{MAE}_{\mathrm{std}}
=
\frac{1}{N}
\sum_{t,h,c}
\left|
\tilde{\hat{y}}_{t,h,c}
-
\tilde{y}_{t,h,c}
\right|
=
\frac{1}{N}
\sum_{t,h,c}
\left|
\frac{\hat{y}_{t,h,c}-y_{t,h,c}}
{\sigma_c^{\mathrm{train}}}
\right|.
\]
These quantities are scale-normalized errors, not min--max normalized scores, and are not constrained to lie in $[0,1]$. A value above one simply means that the forecasting error is larger than the corresponding training-set scale. 

\subsection{Baselines, hyperparameters, and implementation}

The non-adaptive baselines are Autoformer, Informer, DLinear, Linear, NLinear, and frozen SiMBA \citep{2021_autoformerdecompositiontransformers_wu,2021_informerefficienttransformer_zhou,2023_aretransformerseffective_zeng,2024_simbasimplifiedmambabased_patro}. The adaptive comparison uses a shared frozen SiMBA backbone and a shared LoRA forecasting-head adapter. The continuous update adapter implements the inference-time update mechanism described in the Titans-inspired neural-memory literature \citep{2024_titanslearningmemorize_behrouz}, but not the unavailable official architecture. Black-Mamba uses the same adapter and differs only in the event-triggered update policy.

The adapter is a low-rank LoRA module on the forecasting head with rank 4, alpha 4, dropout 0, and validation calibration for one epoch. The online update uses the \texttt{matmul\_b\_only} implementation: the LoRA $A$ matrix is initialized from calibration and kept fixed online, while the $B$ component is updated from streaming gradients. The default leaky accumulator uses integration coefficient $\lambda=0.97$, with trigger threshold calibration from the validation surprisal quantile. Online learning rates are bounded by the validation-normalized update budget, with minimum $10^{-6}$ and maximum $10^{-4}$ in the recorded runs. Backbone training uses Adam with learning rate $10^{-4}$, train/eval batch size 512 in the representative synthetic runs, and effective test batch size 1 for streaming evaluation. Exact per-run configurations are stored in the run summary JSON files.

\subsection{Statistical testing protocol}

The comparison between Black-Mamba and the continuous-update adapter is performed on the same online forecasting stream. This means that, for a fixed dataset and prediction horizon, both methods produce a prediction for exactly the same test windows. We therefore treat the evaluation as a paired comparison: for each test window $t$, we compute the loss of the continuous-update adapter and the loss of Black-Mamba on the same target window, and then compare their difference.

Formally, let
\[
\ell_t^{\mathrm{cont}}
\]
be the loss of the continuous-update adapter on test window $t$, and let
\[
\ell_t^{\mathrm{BM}}
\]
be the corresponding loss of Black-Mamba on the same window. We define the paired loss difference as
\[
d_t = \ell_t^{\mathrm{cont}} - \ell_t^{\mathrm{BM}}.
\]
With this convention, $d_t>0$ means that the continuous-update adapter has larger loss than Black-Mamba on that window, while $d_t<0$ means that the continuous-update adapter is better on that window. The quantity of interest is the average paired difference,
\[
\bar{d} = \frac{1}{T}\sum_{t=1}^{T} d_t.
\]
A positive value of $\bar{d}$ indicates that, on average, continuous updating performs worse than Black-Mamba.

A standard paired $t$-test would assume that the sequence of differences $(d_t)$ is independent across test windows. This assumption is not appropriate for our forecasting traces, because consecutive windows strongly overlap in time and their errors are autocorrelated. To account for this temporal dependence, we use a HAC/Newey--West standard error for $\bar{d}$. The resulting test statistic is
\[
z = \frac{\bar{d}}{\widehat{\mathrm{se}}_{\mathrm{HAC}}(\bar{d})},
\]
where $\widehat{\mathrm{se}}_{\mathrm{HAC}}(\bar{d})$ estimates the uncertainty of the mean difference while allowing nearby windows to be correlated. For the real forecasting experiments, we set the HAC lag to the prediction horizon minus one, so that the standard error accounts for the overlap induced by multi-step forecasting windows.

We test the paired one-sided hypothesis $H_0:\mathbb{E}[d_t]\le 0$ versus $H_1:\mathbb{E}[d_t]>0$, with moving-block bootstrap confidence intervals used only as a robustness check on the sign and scale of the estimated difference.
For synthetic streams, $H=1$, so adjacent predictions have no target-horizon overlap; the protocol therefore reduces to a lag-zero HAC estimate and an iid bootstrap over paired differences.

We use $p<0.05$ as a conventional threshold for local statistical significance. However, we do not interpret p-values in isolation: each result is discussed together with the percentage effect size and the confidence interval. This is important because, on long traces, extremely small differences can become statistically detectable while remaining practically irrelevant. The main text reports only the real-drift p-values used in the central discussion, while full per-horizon results are reported in Appendix~\ref{app:additional-results}.

\subsection{Compute resources}

Compute-intensive runs were executed on a DGX system with 8 NVIDIA A100 40GB GPUs. Most development, trace processing, plotting, and routine analysis were performed on a workstation with an NVIDIA RTX 4070 Ti GPU. The final statistical analyses and figures in this draft were generated from saved CSV/JSON trace artifacts and do not require re-training.

\section{Additional Results}
\label{app:additional-results}

This section reports the additional results supporting the experimental discussion in the main paper. We include full per-horizon benchmark tables, complete paired adaptive-policy comparisons, synthetic trace diagnostics, noise and stress ablations, update-gate sensitivity observations, and the concept-drift diagnostics used to identify the real-drift cases.

\subsection{Full Real Benchmark Tables}
\label{app:full-real-benchmarks}

Table~\ref{tab:real-full-appendix} expands the averaged real benchmark table from the main paper into per-horizon results. The table confirms that most differences among frozen SiMBA, the continuous-update adapter, and Black-Mamba are below standard benchmark-table precision on real datasets. This supports the interpretation used in the main text: on most real benchmarks, online adaptation does not materially change aggregate forecasting accuracy because the test streams are not uniformly drifting. The relevant evidence is therefore not broad real-data dominance, but whether Black-Mamba preserves the backbone while reducing updates and whether it avoids degradation in the strongest real-drift cases.

\begin{table*}[t]
\centering
\caption{Full real benchmark results by dataset and horizon. Each cell reports standardized MSE and MAE; lower is better. Bold and underline mark the best and second-best results across all models; light yellow and light blue mark the best and second-best results among SiMBA, Cont. update, and Black-Mamba, isolating update-policy effects for the chosen frozen baseline. Ties are highlighted within the corresponding rank.}
\label{tab:real-full-appendix}
\scriptsize
\resizebox{\textwidth}{!}{
\begin{tabular}{ll*{8}{cc}}
\toprule
Dataset & Horizon
& \multicolumn{2}{c}{Autoformer}
& \multicolumn{2}{c}{Informer}
& \multicolumn{2}{c}{DLinear}
& \multicolumn{2}{c}{Linear}
& \multicolumn{2}{c}{NLinear}
& \multicolumn{2}{c}{SiMBA}
& \multicolumn{2}{c}{Cont. update}
& \multicolumn{2}{c}{Black-Mamba} \\
\cmidrule(lr){3-4}\cmidrule(lr){5-6}\cmidrule(lr){7-8}\cmidrule(lr){9-10}
\cmidrule(lr){11-12}\cmidrule(lr){13-14}\cmidrule(lr){15-16}\cmidrule(lr){17-18}
& & MSE & MAE & MSE & MAE & MSE & MAE & MSE & MAE
& MSE & MAE & MSE & MAE & MSE & MAE & MSE & MAE \\
\midrule
\multirow{4}{*}{ETTh1}
& 96  & 0.600 & 0.530 & 1.034 & 0.733 & \underline{0.512} & \underline{0.468} & \textbf{0.509} & \textbf{0.467} & 0.525 & 0.469 & \policybest{0.516} & \policybest{0.472} & \policybest{0.516} & \policybest{0.472} & \policybest{0.516} & \policybest{0.472} \\
& 192 & 0.741 & 0.592 & 1.166 & 0.796 & \underline{0.605} & \underline{0.508} & \textbf{0.603} & \textbf{0.507} & 0.631 & 0.516 & \policybest{0.623} & \policybest{0.519} & \policybest{0.623} & \policybest{0.519} & \policybest{0.623} & \policybest{0.519} \\
& 336 & 0.842 & 0.637 & 1.210 & 0.798 & \underline{0.663} & \underline{0.545} & \textbf{0.661} & \textbf{0.542} & 0.721 & 0.562 & \policybest{0.706} & \policybest{0.561} & \policybest{0.706} & \policybest{0.561} & \policybest{0.706} & \policybest{0.561} \\
& 720 & 0.926 & 0.692 & 1.333 & 0.861 & \underline{0.729} & \underline{0.602} & \textbf{0.724} & \textbf{0.596} & 0.808 & 0.617 & \policybest{0.829} & \policybest{0.635} & \policybest{0.829} & \policybest{0.635} & \policybest{0.829} & \policybest{0.635} \\
\midrule
\multirow{4}{*}{ETTh2}
& 96  & 0.322 & 0.386 & 2.086 & 1.103 & 0.342 & 0.403 & 0.335 & 0.395 & \underline{0.273} & \underline{0.345} & \policybest{\textbf{0.265}} & \policybest{\textbf{0.339}} & \policybest{\textbf{0.265}} & \policybest{\textbf{0.339}} & \policybest{\textbf{0.265}} & \policybest{\textbf{0.339}} \\
& 192 & 0.388 & 0.420 & 3.755 & 1.550 & 0.448 & 0.465 & 0.410 & 0.440 & \underline{0.344} & \underline{0.388} & \policybest{\textbf{0.338}} & \policybest{\textbf{0.381}} & \policybest{\textbf{0.338}} & \policybest{\textbf{0.381}} & \policybest{\textbf{0.338}} & \policybest{\textbf{0.381}} \\
& 336 & 0.444 & 0.457 & 4.581 & 1.785 & 0.630 & 0.550 & 0.532 & 0.502 & \underline{0.425} & \underline{0.430} & \policybest{\textbf{0.410}} & \policybest{\textbf{0.425}} & \policybest{\textbf{0.410}} & \policybest{\textbf{0.425}} & \policybest{\textbf{0.410}} & \policybest{\textbf{0.425}} \\
& 720 & 0.601 & 0.545 & 4.564 & 1.747 & 0.842 & 0.639 & 0.971 & 0.688 & \underline{0.564} & \underline{0.504} & \policybest{\textbf{0.557}} & \policybest{\textbf{0.496}} & \policybest{\textbf{0.557}} & \policybest{\textbf{0.496}} & \policybest{\textbf{0.557}} & \policybest{\textbf{0.496}} \\
\midrule
\multirow{4}{*}{ETTm1}
& 96  & 0.661 & 0.552 & 0.693 & 0.590 & \underline{0.412} & \textbf{0.413} & \textbf{0.411} & \textbf{0.413} & 0.422 & \underline{0.419} & \policybest{0.422} & \policybest{0.425} & \policybest{0.422} & \policybest{0.425} & \policybest{0.422} & \policybest{0.425} \\
& 192 & 0.748 & 0.588 & 0.740 & 0.607 & \textbf{0.477} & \textbf{0.443} & \textbf{0.477} & \textbf{0.443} & 0.494 & \underline{0.451} & \policysecond{0.487} & \policybest{0.453} & \policybest{\underline{0.486}} & \policybest{0.453} & \policysecond{0.487} & \policybest{0.453} \\
& 336 & 0.756 & 0.593 & 0.884 & 0.662 & \textbf{0.545} & \textbf{0.476} & \underline{0.546} & \underline{0.478} & 0.572 & 0.488 & \policybest{0.560} & \policybest{0.490} & \policybest{0.560} & \policybest{0.490} & \policybest{0.560} & \policybest{0.490} \\
& 720 & 0.867 & 0.636 & 1.071 & 0.741 & \textbf{0.644} & \textbf{0.524} & \underline{0.646} & \underline{0.526} & 0.693 & 0.542 & \policybest{0.683} & \policybest{0.540} & \policybest{0.683} & \policybest{0.540} & \policybest{0.683} & \policybest{0.540} \\
\midrule
\multirow{4}{*}{ETTm2}
& 96  & 0.197 & 0.302 & 0.509 & 0.579 & 0.169 & 0.277 & 0.165 & \underline{0.270} & \textbf{0.163} & \textbf{0.266} & \policybest{\underline{0.164}} & \policybest{\underline{0.270}} & \policybest{\underline{0.164}} & \policybest{\underline{0.270}} & \policybest{\underline{0.164}} & \policybest{\underline{0.270}} \\
& 192 & 0.237 & 0.326 & 0.792 & 0.725 & 0.237 & 0.333 & 0.223 & 0.319 & \textbf{0.214} & \textbf{0.301} & \policybest{\underline{0.222}} & \policybest{\underline{0.312}} & \policybest{\underline{0.222}} & \policybest{\underline{0.312}} & \policybest{\underline{0.222}} & \policybest{\underline{0.312}} \\
& 336 & 0.301 & 0.367 & 1.096 & 0.812 & 0.348 & 0.413 & 0.305 & 0.381 & \textbf{0.268} & \textbf{0.336} & \policybest{\underline{0.278}} & \policybest{\underline{0.350}} & \policybest{\underline{0.278}} & \policybest{\underline{0.350}} & \policybest{\underline{0.278}} & \policybest{\underline{0.350}} \\
& 720 & 0.376 & 0.409 & 1.960 & 1.091 & 0.459 & 0.474 & 0.519 & 0.504 & \textbf{0.355} & \textbf{0.385} & \policybest{\underline{0.364}} & \policybest{\underline{0.400}} & \policybest{\underline{0.364}} & \policybest{\underline{0.400}} & \policybest{\underline{0.364}} & \policybest{\underline{0.400}} \\
\midrule
\multirow{4}{*}{Electricity}
& 96  & 0.239 & 0.349 & 0.639 & 0.561 & 0.205 & 0.283 & \underline{0.204} & 0.281 & 0.206 & \underline{0.277} & \policybest{\textbf{0.177}} & \policybest{\textbf{0.257}} & \policybest{\textbf{0.177}} & \policybest{\textbf{0.257}} & \policybest{\textbf{0.177}} & \policybest{\textbf{0.257}} \\
& 192 & 0.253 & 0.353 & 0.624 & 0.565 & 0.204 & 0.287 & \underline{0.203} & 0.284 & 0.205 & \underline{0.278} & \policybest{\textbf{0.184}} & \policybest{\textbf{0.265}} & \policybest{\textbf{0.184}} & \policybest{\textbf{0.265}} & \policybest{\textbf{0.184}} & \policybest{\textbf{0.265}} \\
& 336 & 0.268 & 0.363 & 0.551 & 0.516 & \underline{0.221} & 0.302 & 0.222 & 0.304 & 0.225 & \underline{0.295} & \policybest{\textbf{0.205}} & \policybest{\textbf{0.282}} & \policybest{\textbf{0.205}} & \policybest{\textbf{0.282}} & \policybest{\textbf{0.205}} & \policybest{\textbf{0.282}} \\
& 720 & 0.309 & 0.391 & 0.546 & 0.516 & \underline{0.266} & 0.339 & 0.267 & 0.341 & 0.274 & \underline{0.330} & \policybest{\textbf{0.257}} & \policybest{\textbf{0.321}} & \policybest{\textbf{0.257}} & \policybest{\textbf{0.321}} & \policybest{\textbf{0.257}} & \policybest{\textbf{0.321}} \\
\midrule
\multirow{4}{*}{Exchange Rate}
& 96  & 0.365 & 0.423 & 3.449 & 1.499 & 0.232 & 0.331 & \underline{0.213} & \underline{0.310} & \textbf{0.205} & \textbf{0.296} & 0.223 & 0.313 & \policybest{0.220} & \policybest{\underline{0.310}} & \policysecond{0.222} & \policysecond{0.312} \\
& 192 & 0.700 & 0.599 & 3.839 & 1.575 & 0.509 & 0.497 & 0.541 & 0.519 & \textbf{0.425} & \textbf{0.433} & \policybest{\underline{0.456}} & \policysecond{0.448} & \policysecond{0.457} & \policybest{\underline{0.447}} & \policybest{\underline{0.456}} & \policysecond{0.448} \\
& 336 & 1.011 & 0.746 & 4.391 & 1.672 & 1.066 & 0.755 & 0.918 & 0.696 & \textbf{0.702} & \textbf{0.576} & 0.755 & \policysecond{0.597} & \policybest{\underline{0.747}} & \policybest{\underline{0.595}} & \policysecond{0.754} & \policysecond{0.597} \\
& 720 & 1.703 & 1.016 & 5.788 & 1.984 & 3.192 & 1.328 & 3.242 & 1.343 & \textbf{1.303} & \textbf{0.838} & \policybest{\underline{1.610}} & \policybest{\underline{0.925}} & 1.640 & 0.931 & \policysecond{1.611} & \policysecond{0.926} \\
\midrule
\multirow{4}{*}{Traffic}
& 96  & \underline{0.617} & \underline{0.393} & 0.768 & 0.415 & 0.758 & 0.482 & 0.740 & 0.471 & 0.742 & 0.460 & \policybest{\textbf{0.440}} & \policybest{\textbf{0.279}} & \policybest{\textbf{0.440}} & \policybest{\textbf{0.279}} & \policybest{\textbf{0.440}} & \policybest{\textbf{0.279}} \\
& 192 & \underline{0.635} & \underline{0.405} & 0.831 & 0.452 & 0.720 & 0.471 & 0.657 & 0.435 & 0.649 & 0.420 & \policybest{\textbf{0.439}} & \policybest{\textbf{0.278}} & \policybest{\textbf{0.439}} & \policybest{\textbf{0.278}} & \policybest{\textbf{0.439}} & \policybest{\textbf{0.278}} \\
& 336 & 0.717 & 0.460 & 0.940 & 0.507 & 0.714 & 0.462 & 0.679 & 0.440 & \underline{0.654} & \underline{0.417} & \policybest{\textbf{0.444}} & \policybest{\textbf{0.276}} & \policybest{\textbf{0.444}} & \policybest{\textbf{0.276}} & \policybest{\textbf{0.444}} & \policybest{\textbf{0.276}} \\
& 720 & \underline{0.617} & \underline{0.381} & 0.864 & 0.469 & 0.882 & 0.537 & 0.788 & 0.501 & 0.761 & 0.481 & \policybest{\textbf{0.502}} & \policybest{\textbf{0.307}} & \policybest{\textbf{0.502}} & \policybest{\textbf{0.307}} & \policybest{\textbf{0.502}} & \policybest{\textbf{0.307}} \\
\midrule
\multirow{4}{*}{Weather}
& 96  & 0.581 & 0.501 & \textbf{0.488} & \textbf{0.443} & 0.534 & 0.461 & 0.534 & 0.461 & 0.552 & 0.463 & \policybest{\underline{0.524}} & \policybest{\underline{0.446}} & \policybest{\underline{0.524}} & \policybest{\underline{0.446}} & \policybest{\underline{0.524}} & \policybest{\underline{0.446}} \\
& 192 & 0.633 & 0.529 & \textbf{0.546} & \textbf{0.483} & \underline{0.586} & 0.494 & \underline{0.586} & \underline{0.493} & 0.616 & 0.499 & \policybest{0.594} & \policybest{\textbf{0.483}} & \policybest{0.594} & \policybest{\textbf{0.483}} & \policybest{0.594} & \policybest{\textbf{0.483}} \\
& 336 & 0.683 & 0.558 & 0.725 & 0.584 & \textbf{0.619} & \underline{0.514} & \textbf{0.619} & \underline{0.514} & 0.657 & 0.522 & \policybest{\underline{0.633}} & \policybest{\textbf{0.507}} & \policybest{\underline{0.633}} & \policybest{\textbf{0.507}} & \policybest{\underline{0.633}} & \policybest{\textbf{0.507}} \\
& 720 & 0.713 & 0.568 & 0.747 & 0.594 & \textbf{0.669} & \textbf{0.544} & \textbf{0.669} & \textbf{0.544} & 0.723 & 0.558 & \policybest{\underline{0.701}} & \policybest{\underline{0.545}} & \policybest{\underline{0.701}} & \policybest{\underline{0.545}} & \policybest{\underline{0.701}} & \policybest{\underline{0.545}} \\
\bottomrule
\end{tabular}
}
\end{table*}

\subsection{Full Adaptive-Policy Comparisons}
\label{app:full-adaptive-policy}

Table~\ref{tab:adaptive-policy-full-appendix} reports the paired adaptive-policy comparison for the real material-drift datasets and the synthetic drift streams. Positive deltas favor Black-Mamba, while negative deltas favor the continuous-update adapter. The real results show practical equivalence except for Exchange Rate at horizon 720, where continuous updating is worse than Black-Mamba. 

The synthetic results show that continuous updating often obtains slightly lower losses, especially when the drift signal is dense and consistently informative. This is expected: in synthetic streams where almost every new point contributes useful information about the current regime, an always-update policy can extract small additional gains. The key observation is that these gains remain numerically small, while Black-Mamba substantially reduces inference-time writes. This supports the intended trade-off rather than contradicting it: Black-Mamba is not designed to dominate continuous adaptation in every loss cell, but to preserve most of its benefit with fewer and more selective updates.

\begin{table*}[t]
\centering
\caption{Full adaptive-policy comparison between Black-Mamba and the Titans-inspired continuous-update adapter. $\Delta=100\cdot(\mathrm{Cont.}-\mathrm{BM})/\mathrm{BM}$; positive values favor Black-Mamba. $p_{\mathrm{dir}}$ is the HAC one-sided p-value in the direction of the observed mean difference.}
\label{tab:adaptive-policy-full-appendix}
\small
\begin{tabular}{llrrrrr}
\toprule
Dataset / scenario & Horizon & BM update rate & $\Delta$ MSE & $p_{\mathrm{dir}}$ & $\Delta$ MAE & $p_{\mathrm{dir}}$ \\
\midrule
\multirow{4}{*}{Exchange Rate}
& 96  & 7.5\%  & -0.685\% & 0.370 & -0.599\% & 0.270 \\
& 192 & 10.4\% & +0.087\% & 0.420 & -0.160\% & 0.208 \\
& 336 & 9.9\%  & -0.926\% & 0.283 & -0.286\% & 0.371 \\
& 720 & 9.1\%  & +1.797\% & 0.056 & +0.556\% & 0.039 \\
\midrule
\multirow{4}{*}{Traffic}
& 96  & 46.5\% & -0.001\% & $<10^{-3}$ & -0.003\% & $<10^{-3}$ \\
& 192 & 43.8\% & -0.000\% & 0.002 & -0.001\% & $<10^{-3}$ \\
& 336 & 42.5\% & -0.000\% & 0.007 & -0.001\% & $<10^{-3}$ \\
& 720 & 46.3\% & -0.000\% & $<10^{-3}$ & -0.000\% & $<10^{-3}$ \\
\midrule
L-Gradual & 1 & 53.4\% & -0.329\% & $<10^{-3}$ & -0.353\% & $<10^{-3}$ \\
L-Gradual noisy & 1 & 49.2\% & -0.408\% & $<10^{-3}$ & -0.287\% & $<10^{-3}$ \\
Recurrent seasonal & 1 & 40.9\% & -0.847\% & $<10^{-3}$ & -0.502\% & $<10^{-3}$ \\
Recurrent seasonal noisy & 1 & 38.6\% & -0.454\% & $<10^{-3}$ & -1.064\% & $<10^{-3}$ \\
Regime plateau & 1 & 56.6\% & -0.360\% & $<10^{-3}$ & -0.391\% & $<10^{-3}$ \\
Regime plateau noisy & 1 & 48.1\% & -0.315\% & $<10^{-3}$ & -0.231\% & $<10^{-3}$ \\
Selective stress & 1 & 78.4\% & +0.114\% & $<10^{-3}$ & -0.105\% & $<10^{-3}$ \\
\bottomrule
\end{tabular}
\end{table*}

\subsection{Synthetic Update-Efficiency Diagnostics}
\label{app:synthetic-update-efficiency}

Table~\ref{tab:synthetic-selectivity-diagnostics} reports diagnostics that are not captured by the aggregate adaptive-policy table. It measures whether Black-Mamba allocates updates selectively across the stream. For each synthetic scenario, we compare the Black-Mamba update rate in the easiest and hardest segment quartiles, where segment difficulty is measured by segment MSE. The positive correlations and hard/easy update-rate ratios show that the gate is not merely reducing updates uniformly: it increases update intensity when the stream becomes harder.

The heatmaps in Figure~\ref{fig:heatmaps} visualize the same behavior over time. The update-rate heatmap shows that Black-Mamba remains sparse in easier regions but increases update frequency in difficult segments. The relative-MSE heatmap shows that these difficult regions are not uniformly distributed across the stream, which is precisely the setting in which a selective policy is meaningful: the model should not pay the cost of continuous adaptation when the stream is locally stable, but should become more active when mismatch accumulates.

The last two columns of Table~\ref{tab:synthetic-selectivity-diagnostics} report how often Black-Mamba has lower MSE than the continuous-update adapter at segment and rolling-window level. These counts are especially informative for the selective-update stress stream, where Black-Mamba wins 18/20 segments and 384/449 rolling windows. This supports the intended mechanism: transient or corrupted errors need not be written immediately, while persistent mismatch still accumulates enough evidence to trigger adaptation.

Finally, Figure~\ref{fig:baseline-capture} relates update rate to the fraction of the continuous-update gain over frozen SiMBA captured by Black-Mamba. Across the six synthetic scenarios with the frozen SiMBA trace, Black-Mamba captures $98.9\%$ of the MSE improvement obtained by continuous updating while using substantially fewer updates. This is the most direct efficiency interpretation of the synthetic results: Black-Mamba does not remove the benefit of test-time adaptation; it obtains almost all of it with a much smaller write budget.

\begin{table*}[t]
\centering
\caption{Synthetic selectivity diagnostics. ``Easy'' and ``hard'' refer to the lowest- and highest-MSE segment quartiles. BM wins count segments or rolling windows where Black-Mamba has lower MSE than the continuous-update adapter.}
\label{tab:synthetic-selectivity-diagnostics}
\small
\resizebox{\textwidth}{!}{%
\begin{tabular}{lrrrrrr}
\toprule
Scenario & Easy update rate & Hard update rate & Hard/easy ratio & Corr. MSE/update & BM segment wins & BM rolling wins \\
\midrule
L-Gradual & 39.0\% & 71.7\% & 1.84$\times$ & +0.960 & 2/20 & 86/449 \\
L-Gradual noisy & 36.7\% & 64.1\% & 1.74$\times$ & +0.989 & 1/20 & 47/449 \\
Recurrent seasonal & 35.4\% & 47.7\% & 1.35$\times$ & +0.887 & 2/20 & 99/449 \\
Recurrent seasonal noisy & 36.8\% & 40.6\% & 1.10$\times$ & +0.689 & 6/20 & 150/449 \\
Regime plateau & 41.3\% & 74.4\% & 1.80$\times$ & +0.986 & 0/20 & 42/449 \\
Regime plateau noisy & 39.6\% & 59.2\% & 1.50$\times$ & +0.995 & 7/20 & 125/449 \\
Selective update stress & 43.0\% & 98.3\% & 2.29$\times$ & +0.823 & 18/20 & 384/449 \\
\bottomrule
\end{tabular}}
\end{table*}

\begin{figure*}[t]
\centering
\begin{minipage}{0.49\textwidth}
\centering
\includegraphics[width=\linewidth]{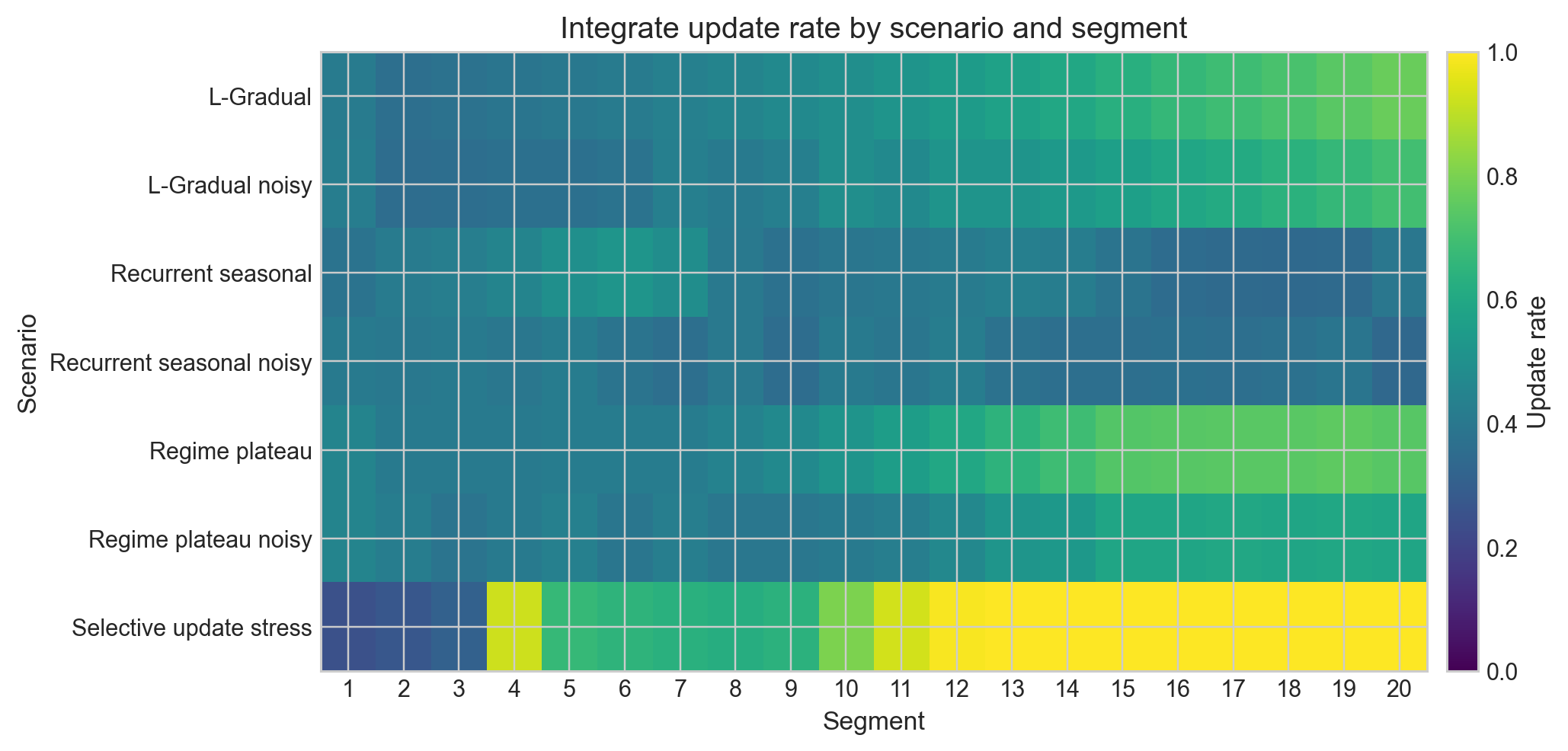}
\end{minipage}
\hfill
\begin{minipage}{0.49\textwidth}
\centering
\includegraphics[width=\linewidth]{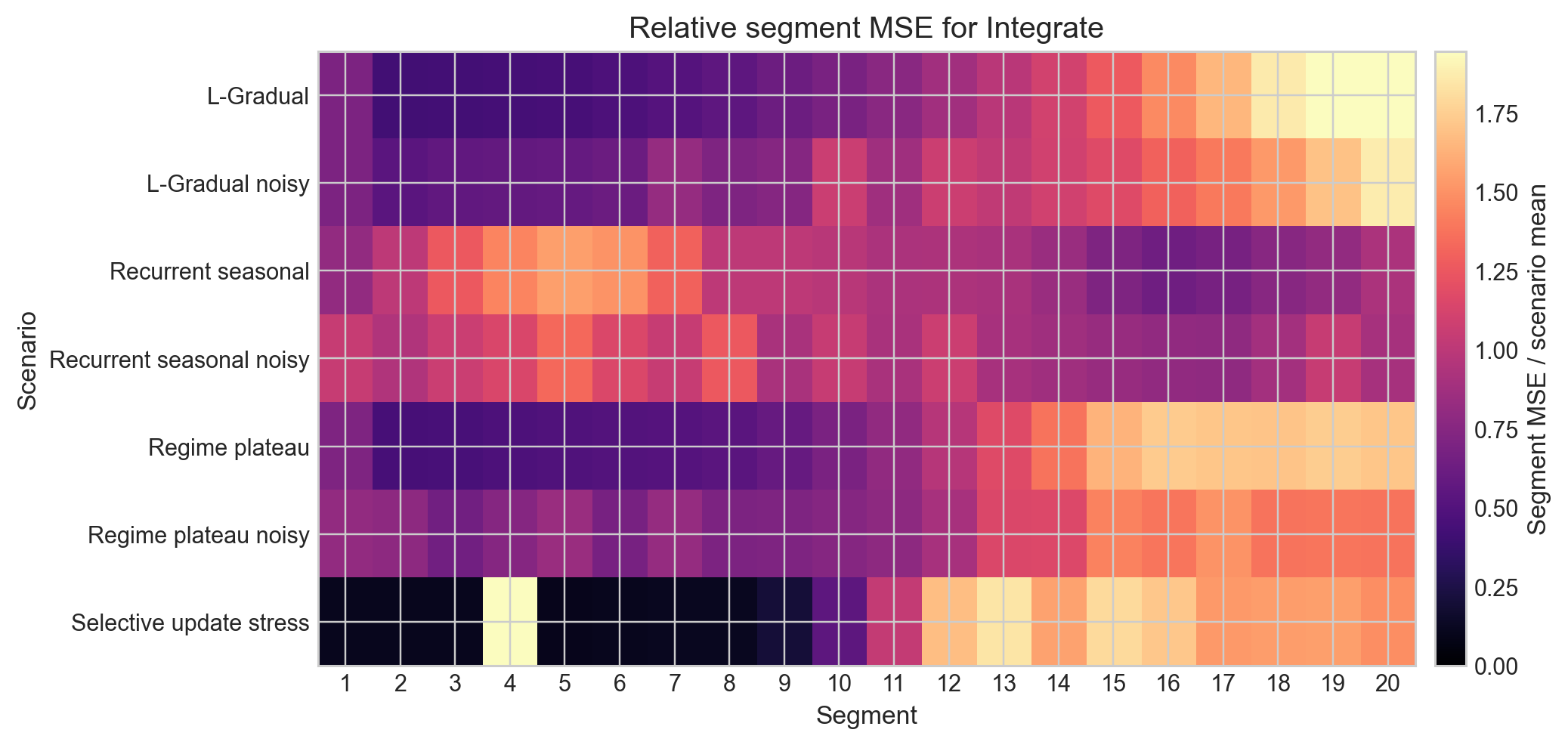}
\end{minipage}
\caption{Black-Mamba temporal diagnostics. Left: update rate by scenario and segment. Right: relative segment MSE. The two panels show that update intensity tends to rise where the stream becomes harder.}
\label{fig:heatmaps}
\end{figure*}

\begin{figure*}[t]
\centering
\includegraphics[width=0.52\textwidth]{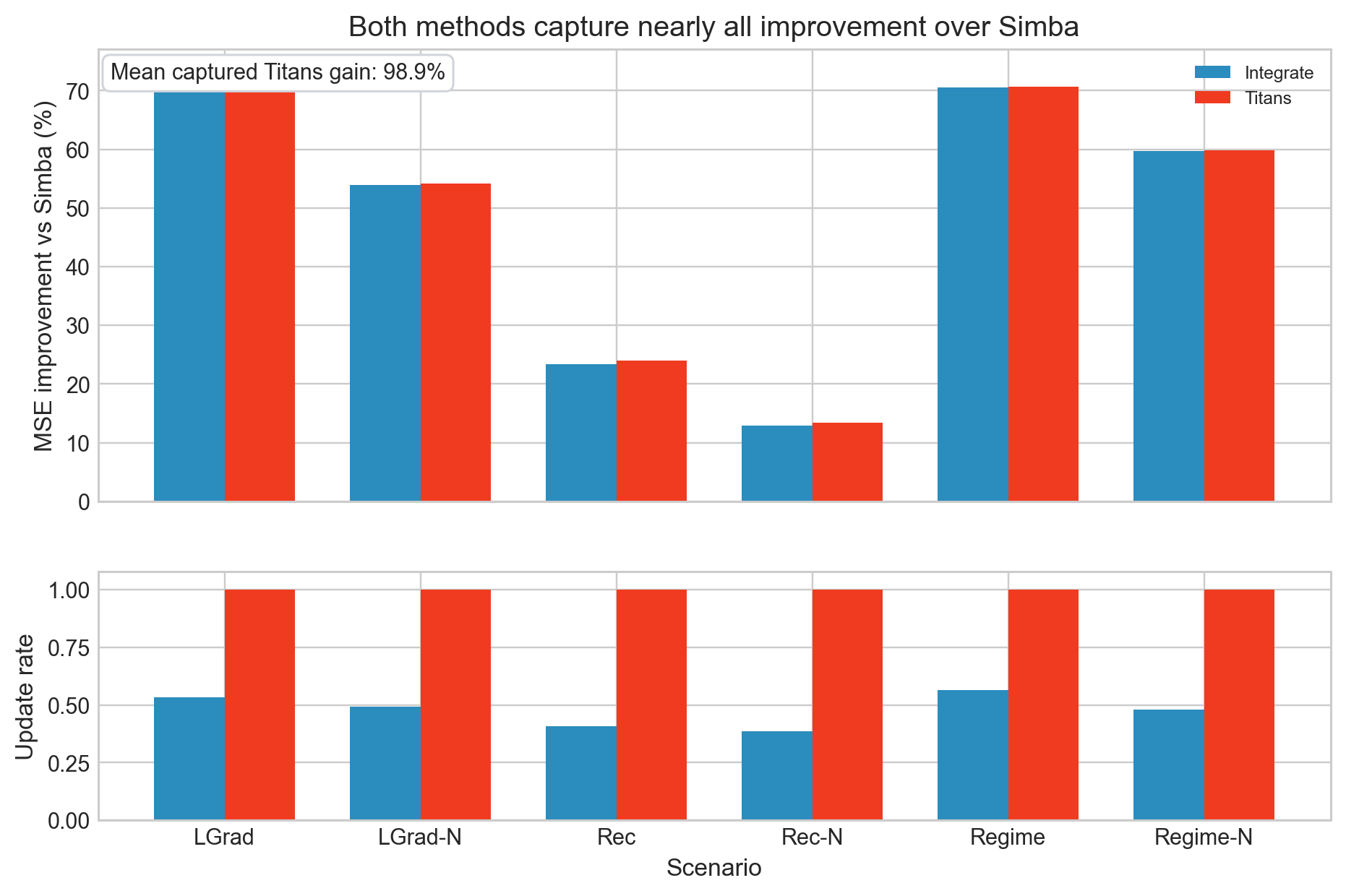}
\caption{Baseline-capture diagnostic. In the six synthetic scenarios with a frozen SiMBA trace in the analysis index, Black-Mamba captures on average 98.9\% of the MSE improvement obtained by continuous updating over frozen SiMBA, while using substantially fewer updates.}
\label{fig:baseline-capture}
\end{figure*}

\subsection{Update-Policy Hyperparameters}
\label{app:update-policy-hyperparameters}

Black-Mamba is adaptive at inference time, but its update policy still depends on hyperparameters that define how evidence is accumulated and when an update is triggered. In our experiments, these parameters are not fixed from arbitrary constants alone: the trigger calibration is selected on the validation stream, while the leak controls the time scale over which surprisal evidence persists. This is important because the update gate is itself part of the deployed system. It determines which errors are treated as transient noise and which errors are interpreted as evidence of persistent drift.

This also introduces a slower form of aging. A frozen forecasting model can become obsolete when the data-generating process changes, but an adaptive model can still inherit assumptions from the validation period through its update policy. If the validation stream underestimates or overestimates the drift intensity, noise level, or regime persistence of the deployment stream, the gate may become too conservative or too reactive. In this sense, Black-Mamba reduces the need to continuously retrain the forecasting model, but it does not remove the need to monitor whether the update policy remains calibrated to the operational domain.

This limitation is also an opportunity for domain-specific deployment. In domains where practitioners have prior knowledge about drift time scales, such as network traffic, recommender systems, information retrieval logs, energy demand, financial streams, or sensor networks, the leak and trigger calibration can be tied to meaningful operational frequencies. For example, one may ask whether the update-policy calibration should be refreshed daily, weekly, after detected regime changes, or only when validation-style monitoring shows that the gate is under- or over-triggering. This refresh schedule is likely to be much cheaper than full model retraining, but it should still be studied explicitly.

We therefore view update-policy calibration as a separate layer of adaptation. The experiments in this paper show that a validation-calibrated gate can reduce inference-time writes while preserving most of the benefit of continuous updating, but they do not exhaust the design space of time-varying gates. Future work should study how leak, trigger thresholds, and surprisal calibration evolve under long deployments, and how often they must be refreshed relative to full backbone retraining.

\subsection{Noise and Stress Ablations}
\label{app:noise-stress-ablation}

The noisy synthetic variants ablate the reliability of the drift signal by adding noise to the gradual, recurrent seasonal, and regime-plateau streams. In these settings, the continuous-update adapter remains slightly better than Black-Mamba in most MSE and MAE comparisons, consistent with the data construction: noise perturbs the stream, but the underlying drift signal remains dense and continuously informative. Since the continuous adapter writes at nearly every step, it can exploit frequent small signals more aggressively than a thresholded policy.

This does not invalidate the proposed update rule. The claim is not that selective updating must always minimize raw loss when every observation is useful, but that continuous updating pays for marginal gains with a much larger update budget and can become vulnerable when instantaneous errors are unreliable evidence of persistent drift. The noisy ablations sharpen this trade-off: even when continuous updating benefits from dense synthetic supervision, Black-Mamba remains close in accuracy while using substantially fewer updates.

The selective-update stress stream tests the opposite regime: long easy regions, transient corrupted bursts, and later persistent difficulty. This setting matches the failure mode targeted by Black-Mamba, where not every large error should be written into the adapter. The result supports leaky surprisal accumulation: transient corruptions can decay without permanent updates, while persistent mismatch still accumulates enough evidence to trigger adaptation.

\subsection{Concept-Drift Diagnostics}
\label{app:concept-drift-diagnostics}

The drift diagnostics explain why the main text does not claim broad real-data superiority from online adaptation. We use a model-based concept-transfer probe: a reference predictor $f_0$ is trained on the initial regime, while a window-local predictor $f_w$ is trained on the first half of a later window and evaluated against $f_0$ on the second half. For forecasting windows, we define
\[
D_w=\max\left(0,\frac{\mathrm{RMSE}(f_0,w)}{\mathrm{RMSE}(f_0,\mathrm{val})}-1\right),
\qquad
G_w=\max\left(0,\frac{\mathrm{RMSE}(f_0,w)-\mathrm{RMSE}(f_w,w)}{\mathrm{RMSE}(f_0,w)}\right),
\]
and report $\mathrm{concept\_score}_w=\max(D_wG_w,G_w)$. Thus, drift evidence increases when the train-regime predictor degrades and a local predictor recovers loss on the same window.

Among the real forecasting datasets, only Exchange Rate and Traffic are classified as material concept drift. Exchange Rate is the most severe case, with 6 strong windows and maximum concept score 18.410; Traffic has 7 strong windows but a smaller maximum concept score of 1.811. ETTh1 is weak or inconclusive, while the remaining real datasets are classified as no material concept drift.

\begin{table}[h]
\centering
\caption{Real-dataset drift diagnostics.}
\label{tab:real-drift}
\begin{tabular}{lrrrr}
\toprule
Dataset & Verdict & Strong & Weak & Max score \\
\midrule
Exchange Rate & material & 6 & 0 & 18.410 \\
Traffic & material & 7 & 0 & 1.811 \\
ETTh1 & weak/inconclusive & 1 & 0 & 0.561 \\
ETTh2 & no material & 0 & 0 & 0.005 \\
ETTm1 & no material & 0 & 0 & 0.346 \\
ETTm2 & no material & 0 & 0 & 0.594 \\
Electricity & no material & 0 & 1 & 0.159 \\
Weather & no material & 0 & 0 & 0.000 \\
\bottomrule
\end{tabular}
\end{table}

All seven synthetic datasets contain drift by construction; five are material, while the two recurrent variants are weak/inconclusive because recurring regimes reduce persistent degradation.

\end{document}